%% file: complexity_and_approximation_of_the_fuzzy_kmeans_problem.tex
\documentclass[10pt]{llncs}

\input{header_packages}
\input{header_format}

\input{header_commands}

\input{title.tex}

\begin{document}
	\maketitle
	\input{abstract.tex}
	\input{introduction.tex}
	\input{contribution.tex}
	\input{techniques.tex}

	\input{algorithm.tex}

	\input{conclusion.tex}

	\input{appendix}

	\bibliographystyle{splncs}

\input{complexity_and_approximation_of_the_fuzzy_kmeans_problem.bbl}
\end{document}

%% file: header_packages.tex
\usepackage{graphicx}

\usepackage[ruled, vlined, linesnumbered]{algorithm2e}

\usepackage{amsfonts,amssymb,mathtools}
\usepackage{verbatim}
\usepackage{booktabs}
\usepackage{nicefrac}

\usepackage{url}

\usepackage{framed}
\usepackage{color}

%% file: header_format.tex
\usepackage{microtype}

\usepackage{geometry}
\newtheorem{observation}{Observation}

%% file: header_commands.tex
\newcommand{\N}{\mathbb{N}}
\newcommand{\R}{\mathbb{R}}
\newcommand{\Q}{\mathbb{Q}}

\newcommand{\E}{\text{E}}
\newcommand{\Var}{\text{Var}}

\newcommand{\cE}{\mathcal{E}}
\newcommand{\cO}{\mathcal{O}}
\newcommand{\cU}{\mathcal{U}}
\newcommand{\cG}{\mathcal{G}}

\newcommand{\abs}[1]{\left\vert #1 \right\vert}
\newcommand{\norm}[1]{\left\Vert #1 \right\Vert_2}
\newcommand{\sca}[2]{\left\langle #1, #2 \right\rangle}
\newcommand{\OneTo}[1]{\in[#1]}

\makeatletter
\newcommand{\subalign}[1]{%
  \vcenter{%
    \Let@ \restore@math@cr \default@tag
    \baselineskip\fontdimen10 \scriptfont\tw@
    \advance\baselineskip\fontdimen12 \scriptfont\tw@
    \lineskip\thr@@\fontdimen8 \scriptfont\thr@@
    \lineskiplimit\lineskip
    \ialign{\hfil$\m@th\scriptstyle##$&$\m@th\scriptstyle{}##$\crcr
      #1\crcr
    }%
  }
}
\makeatother

\newcommand{\texorpdfstring}[2]{#1}

\newcommand{\Lkj}{L_{k,j}}

\newcommand{\rnk}{r_{nk}}
\newcommand{\rnl}{r_{nl}}

\newcommand{\znk}{z_{nk}}
\newcommand{\Rmin}{\mathcal{B}} 
\newcommand{\wmin}{w_{\min}}
\newcommand{\wmax}{w_{\max}}

\newcommand{\rnkSet}{\{\rnk\}_{n,k}}
\newcommand{\rNKSet}{\{\rnk\}_{n\OneTo{N},k\OneTo{K}}}

\newcommand{\rNLSet}{\{\rnl\}_{n\OneTo{N},l\OneTo{L}}}
\newcommand{\rnSet}{\{\rnk\}_{n}}

\newcommand{\muSet}{\{\mu_k\}_{k}}
\newcommand{\tmuSet}{\{\tilde\mu_k\}_{k}}
\newcommand{\tmuKSet}{\{\tilde\mu_k\}_{k\OneTo{K}}}

\newcommand{\muKSet}{\{\mu_k\}_{k\OneTo{K}}}

\newcommand{\tmu}{\tilde\mu}
\newcommand{\tC}{\tilde C}

\newcommand{\phiopt}{\phi_{(X,K,m)}^{OPT}}

\DeclareMathOperator{\km}{km}
\newcommand{\phim}{\phi^{(m)}}
\newcommand{\phiXm}{\phi_X^{(m)}}
\newcommand{\phiXmk}{\phi_{X,k}^{(m)}}
\newcommand{\phiXml}{\phi_{X,l}^{(m)}}

\newcommand{\Gal}{\operatorname{Gal}}

\DeclareMathOperator*{\ball}{\mathcal{B}}
\DeclareMathOperator*{\vol}{vol}

\DeclareMathOperator*{\sgn}{sgn}

%% file: title.tex
\title{Complexity and Approximation\\of the Fuzzy $K$-Means Problem}
\titlerunning{Complexity and Approximation of the Fuzzy $K$-Means Problem} 

\author{Johannes Bl\"omer \and Sascha Brauer \and Kathrin Bujna}
\authorrunning{J. Bl\"omer \and S. Brauer, \and K. Bujna} 

\institute{
Paderborn University, 33098 Paderborn, Germany,\\
\email{\{bloemer, sascha.brauer, kathrin.bujna\}@uni-paderborn.de}
}

%% file: abstract.tex
\begin{abstract}
The fuzzy $K$-means problem is a generalization of the classical $K$-means problem to soft clusterings, i.e. clusterings where each points belongs to each cluster to some degree. 
Although popular in practice, prior to this work the fuzzy $K$-means problem has not been studied from a complexity theoretic or algorithmic perspective.
We show that optimal solutions for fuzzy $K$-means cannot, in general, be expressed by radicals over the input points. 
Surprisingly, this already holds for very simple inputs in one-dimensional space.
Hence, one cannot expect to compute optimal solutions exactly.
We give the first $(1+\epsilon)$-approximation algorithms for the fuzzy $K$-means problem. 
First, we present a deterministic approximation algorithm whose runtime is polynomial in $N$ and linear in the dimension $D$ of the input set, given that $K$ is constant, i.e. a polynomial time approximation algorithm given a fixed $K$. 
We achieve this result by showing that for each soft clustering there exists a hard clustering with comparable properties. 
Second, by using techniques known from coreset constructions for the $K$-means problem, we develop a  deterministic approximation algorithm that runs in time almost linear in $N$ but exponential in the dimension $D$. 
We complement these results with a randomized algorithm which imposes some natural restrictions on the input set and whose runtime is comparable to some of the most efficient approximation algorithms for $K$-means, i.e. linear in the number of points and the dimension, but exponential in the number of clusters. 
\end{abstract}

%% file: introduction.tex
\section{Introduction}\label{sec:intro}
Clustering is a widely used technique in unsupervised machine learning. 
Simply speaking, its goal is to group similar objects. 
This problem occurs in a wide range of practical applications in many fields such as image analysis, information retrieval, and bioinformatics. 
We call a grouping of objects into a given number of clusters a \emph{hard clustering} if each object is assigned to exactly one cluster.
A popular example for a hard clustering problem is the well known $K$-means problem. 
In contrast, in a \emph{soft clustering} each object belongs to each cluster with a certain degree of membership.
There is a continuous generalization of the $K$-means problem that leads to a such a soft clustering problem, known as the \emph{fuzzy $K$-means} problem.

\input{introduction_fuzzy.tex}
\input{introduction_relatedwork.tex}
\input{introduction_overview.tex}

%% file: introduction_fuzzy.tex
\subsection{Fuzzy \texorpdfstring{$K$}{K}-Means}
\cite{dunn73} was the first to present a fuzzy $K$-means objective function, which was later extended by  \cite{bezdek84}. 
Today, fuzzy $K$-means has found a wide range of practical applications, for example in image segmentation \cite{rezaee} and biological data analysis \cite{dembele}.

\subsubsection{Fuzzy \texorpdfstring{$K$}{K}-Means Problem}
Let $X =\{(x_1,w_1),\ldots, (x_N,w_N) \}$ be a set of data points $x_n\in\R^D$ weighted by $w_n\in \R_{\geq 0}$. 
We want to group $X$ into some predefined number of clusters $K$.
These clusters are represented by mean vectors $\{\mu_1,\ldots,\mu_K\}\subset\R^D$.
In a fuzzy clustering, each data point $x_n$ belongs to each cluster, represented by a $\mu_k$, with a certain membership value $r_{nk}\in[0,1]$.
The {fuzzy $K$-means problem} has an additional parameter, the so-called fuzzifier $m\in \N_{>1}$, which is chosen in advance and is not subject to optimization.
In simple terms, the fuzzifier $m$ determines how much clusters are allowed to overlap, i.e. how soft the clustering is.
\begin{problem}[Fuzzy $K$-means]
  Given $X =\{(x_n,w_n)\}_{n\OneTo{N}} \subset \R^D\times \R_{\geq 0}$, $K\geq 1$ and $m\geq 2$, find $C=\{ \mu_k \}_{k\OneTo{K}} \subset \R^D$ and $R=\{ r_{nk} \}_{n\OneTo{N},k\OneTo{K}}\subset[0,1]$ minimizing
\[  
\phi_{X}^{(m)}(C,R)=\sum_{n=1}^N\sum_{k=1}^K r_{nk}^m w_n \norm{ x_n - \mu_k }^2\enspace, 
\]
 subject to: 
 $ \sum_{k=1}^K r_{nk} = 1$ for all $n\OneTo{N}$.
 
 We denote the costs of an optimal solution by $\phiopt$. 
\end{problem}
For $m=1$ this problem would coincide with the classical $K$-means problem, while for $m\rightarrow \infty$ the memberships converge to a uniform distribution and the centers converge to the center of the data set $X$.
Our problem definition is a generalization of the original definition presented in \cite{bezdek84} in that we consider weighted data sets. 
By setting all weights to $1$, we obtain the original definition.

\subsubsection{Fuzzy \texorpdfstring{$K$}{K}-Means (FM) Algorithm}
The most widely used heuristic for the fuzzy $K$-means problem is an alternating optimization algorithm known as \emph{fuzzy $K$-means (FM) algorithm}. 
It is defined by the following first-order optimality conditions \cite{bezdeck87}: 
Fixing the means $\{\mu_k\}_{k\OneTo{K}}$, optimal memberships are given by 
\begin{align}\label{eq:opt-membership} 
  r_{nk} = \frac{\norm{x_n - \mu_k}^{-\frac{2}{m-1}}}{\sum_{l=1}^K \norm{x_n - \mu_l}^{-\frac{2}{m-1}}}
\end{align}
if $x_n\neq\mu_l$ for all $l\OneTo{K}$. 
If $x_n$ coincides with some of the $\mu_l$, then the membership of $x_n$ can be distributed arbitrarily among those $\mu_l$ with $\mu_l=x_n$.
Fixing the memberships $\{r_{nk}\}_{n,k}$, the optimal means are given by 
\begin{align} 
 \mu_k = \frac{\sum_{n=1}^N r_{nk}^m w_n x_n}{\sum_{n=1}^N r_{nk}^m w_n}\enspace . \label{eq:opt-means}
\end{align} 
A major downside of the FM algorithm is that there are no guarantees on the quality of computed solutions.
\begin{observation}\label{obs:poor-solution:intro}
The FM algorithm converges to a  (local) minimum or a saddle point that can be \emph{arbitrarily poor} compared to an optimal solution. 
Such points can be reached by the FM algorithm even if it is initialized with points from the given point set 
\end{observation}
A proof of this observation can be found in Section~\ref{sec:poor-solutions}.

%% file: introduction_relatedwork.tex
\subsection{Related Work}
Although the fuzzy $K$-means problem appears in a wide range of practical applications, so far there has been no complexity classification. 
To the best of our knowledge, there are no hardness results for the fuzzy $K$-means problem.
It is not even known whether it lies in NP.
The same holds for other soft clustering problems, such as the maximum-likelihood estimation problem for Gaussian mixture models \cite{bishop06} or the soft-clustering problem \cite{mackay03}.

Two problems that are closely related to the fuzzy $K$-means problem are the $K$-means and the $K$-median problem.
The complexity of the $K$-means problem is well-studied.
For fixed $K$ and $D$, there is a polynomial time algorithm solving the problem optimally \cite{inaba94}.
The $K$-means problem is NP-complete, even if $K$ or $D$ is fixed to $2$ \cite{dasgupta08} \cite{vattani}.
Furthermore, assuming P$\neq$NP, there is no PTAS for the $K$-means problem for arbitrary $K$ and $D$ \cite{awasthi15}.
However, there are several approximation algorithms known, such as a polynomial-time constant-factor approximation algorithm \cite{kanungo02} and a $(1+\epsilon)$-approximation algorithm with runtime polynomial in $N$ and $D$ \cite{kumar04}. 
The $K$-median problem is a variant of the $K$-means problem that uses the Euclidean instead of the squared Euclidean distance. 
Just as the $K$-means problem, the $K$-median problem is NP-hard, even for $D=2$ \cite{megiddoS84}. 
However, it is known that the optimal solutions to the $K$-median problem have an inherently different structure than the solutions to the $K$-means problem.
Even in the plane, optimal solutions of the $1$-median problem are in general not expressable by radicals over $\Q$ \cite{bajaj88}.

Many practical applications make use of the fuzzy $K$-means (FM) algorithm, which does not yield any approximation guarantees. 
However, \cite{bezdek84} and \cite{bezdeck87} proved convergence of the FM algorithm to a local minimum or a saddle point of the objective function.
Among others, \cite{hoeppner03} and \cite{kim88} address the problem of determining and distinguishing whether the algorithm has reached a local minimum or a saddle point.
Furthermore, it is known that the algorithm converges locally, i.e. started sufficiently close to a minimizer, the iteration sequence converges to that particular minimizer \cite{hathaway86}.
However, to the best of our knowledge, there are no theoretical results on approximation algorithms for the fuzzy $K$-means problem.

%% file: introduction_overview.tex
\subsection{Overview}
The following technical part of the paper is divided in three parts.
In Section~\ref{sec:contrib} we give an overview on our results. 
In Section~\ref{sec:contrib:complexity} we formally state our result that the fuzzy $K$-means problem is not solvable by radicals.
In Section~\ref{sec:contrib:algorithms} we present our results on approximation algorithms.
In Section~\ref{sec:techniques} give an overview on our algorithmic techniques.
In Section~\ref{sec:sketch} we outline the analysis of our approximation algorithms. 
The interested reader can find fully detailed proofs in Section~\ref{sec:appendix}.

%% file: contribution.tex
\section{Our Contribution}\label{sec:contrib}
We initiate the complexity theoretical and algorithmic study of the fuzzy $K$-means problem.

\input{contribution_complexity.tex}
\input{contribution_algorithms.tex}

%% file: contribution_complexity.tex
\subsection{Complexity}\label{sec:contrib:complexity}
We say that a fuzzy $K$-means solution is not solvable by radicals if neither means nor memberships can be expressed in terms of $(+, -, \cdot, /, \sqrt[q]{\hspace*{11pt}})$ over the domain of the input. 

\begin{theorem}\label{thm:radicals}
The fuzzy $K$-means problem for $m = 2$, $K = 2$, $D\geq 1$, $X\subset\N$ and $\abs{X}\geq 6$ is in general not optimally solvable by radicals over $\Q$. 
That is, neither the coordinates of the mean vectors nor the membership values can be expressed in terms of $(+, -, \cdot, /, \sqrt[q]{\hspace*{11pt}})$.
\end{theorem}

This result is an application of the technique used by Bajaj \cite{bajaj88} who proved the same result for the $K$-median problem.
Notably our result already holds for $m=2$ and in one-dimensional space. 
For instance, we show that an optimal solution of the fuzzy $2$-means problem (with $m=2$) for the set $X=\{-3,-2,-1,1,2,3\}$ is not solvable by radicals over $\Q$. 
In contrast, the $K$-means and $K$-median problem can both even be solved efficiently for $D=1$.
As for $m=2$, it is noteworthy that in this case the first-order optimality conditions for means and memberships (cf. Equations~\eqref{eq:opt-membership} and \eqref{eq:opt-means}) lead to rationals in the input domain, respectively.
A consequence of the inexpressibility by radicals is that no algorithm can solve the fuzzy $K$-means problem optimally if it only uses arithmetic operations and root extraction to obtain the zeroes of an algebraic equation.
A more detailed discussion of the implications of unsolvability by radicals can be found in \cite{bajaj88}.

%% file: contribution_algorithms.tex
\subsection{Approximation Algorithms}\label{sec:contrib:algorithms}
We present the first $(1+\epsilon)$-approximation algorithms for the fuzzy $K$-means problem.

\subsubsection{A PTAS For Fixed \texorpdfstring{$K$}{K} and \texorpdfstring{$m$}{m}}
We present the first PTAS for the fuzzy $K$-means problem, assuming a constant number of clusters $K$ and a constant fuzzifier $m$.
That is, for any given $\epsilon \in[0,1]$, our algorithm computes an $(1+\epsilon)$-approximation to the fuzzy $K$-means problem in time polynomial in the number of points $N$ and dimension $D$.

\begin{theorem}\label{thm:sfm:det}
	There is a deterministic algorithm that, given $X = \{(x_n,w_n)\}_{n\OneTo{N}}\subset\R^D\times\Q_{\geq 0}$, $K\in\N$, $m\in\N$, and $\epsilon\in(0,1]$, computes a solution $(C,R)$ such that 
	\[ \phiXm(C,R) \leq \left(1+\epsilon\right) \phiopt\ . \]
	The algorithms' runtime is bounded by
	$D\cdot N^{\mathcal{O}\left( K^2 \log(K) \cdot \frac{1}{\epsilon} \left( m \log\left(\frac{m}{\epsilon}\right) +   \log\left(\frac{\wmax}{\wmin}\right)\right) \right)}$, 
	where $\wmax = \max_{n\OneTo{N}}w_n$ and $\wmin = \min_{n\OneTo{N}}w_n$. 
\end{theorem} 

The main idea behind our result is to exploit the existence of a hard clusters that exhibit characteristics similar to those of a fuzzy clusters. 
By combining this result with a sampling technique which is well known from the $K$-means problem and applying exhaustive search, we obtain the algorithm.

\subsubsection{A Fast Deterministic \texorpdfstring{$(1+\epsilon)$}{(1+eps)}-Approximation Algorithm}
By using a completely different technique, we obtain a deterministic algorithm whose runtime almost linearly depends on $N$. 
On the negative side, we have to give up the linear dependence on the dimension $D$ for this. 

\begin{theorem}\label{thm:candidate-means-algo}
	There is a deterministic algorithm that, given $X = \{x_n\}_{n\OneTo{N}}\subset\R^D$, $K\in\N$, $m\in\N$, and $\epsilon\in(0,1]$, computes a solution $(C,R)$ such that
	\[ \phiXm(C,R) \leq (1+\epsilon) \phiopt\ . \]
	The algorithms' runtime is bounded by $N\left(\log(N)\right)^K K^{\cO\left(K^2 D \log(1/\epsilon) m \right)}$.
\end{theorem}

The runtime of this algorithm is not comparable with the runtime of the algorithm from Theorem~\ref{thm:sfm:det}. 
However, comparing the terms $N\log(N)^K$ and $N^{\cO(K^2\log(K))}$ with one another, we find that the runtime of the algorithm from Theorem~\ref{thm:sfm:det} depends much stronger on $K$ than the runtime of our algorithm from Theorem~\ref{thm:candidate-means-algo}. 
For instance, assuming $K^2 D=\cO\left({\log(N)}/{\log\log(N)}\right)$,  the runtime of our algorithm from Theorem~\ref{thm:candidate-means-algo} is still polynomial in $N$, i.e. $N^{\cO\left(\log\left(\frac{1}{\epsilon}\right) \right)}$, assuming that $m$ is constant. 
In comparison, the runtime of our PTAS from Theorem~\ref{thm:sfm:det} would then be exponential in $N$. 
Hence, Theorem~\ref{thm:sfm:det} and Theorem~\ref{thm:candidate-means-algo} complement each other. 

The idea behind our algorithm from Theorem~\ref{thm:candidate-means-algo} is the same as behind the coreset construction of \cite{HarPeled03} as it is used by \cite{Chen09}. 
That is, we construct a small set of good candidate means. 
After generating this candidate set, our algorithm simply tests all these candidates and chooses the best one.

\subsubsection{A Fast Randomized \texorpdfstring{$(1+\epsilon)$}{(1+eps)}-Approximation Algorithm}
Last, we show that there is a randomized algorithm with runtime linear in $N$ and $D$. 
However, in return for this speedup, this algorithm has some requirement on the input sets. 
More precisely, our algorithm from Theorem~\ref{thm:sfm:rand} approximates the best $\left(\alpha \sum_{n=1}^N w_n,\ K\right)$-balanced solution. 

\begin{definition}[$(\Rmin,K)$-balanced]\label{def:intro:r-balanced}
 Let $X = \{(x_n,w_n)\}_{n\OneTo{N}}\subset\R^D\times\R_{\geq 0}$. 
 Given a solution with memberships $R= \{\rnl\}_{n\OneTo{N},l\OneTo{L}}$, we denote the {weight of the $l^{th}$ fuzzy cluster} by 
\begin{align} 
   R_l \coloneqq \sum_{n=1}^N \rnl^m w_n\ . \label{eq:Rk}
\end{align}
We say that the memberships $R$ are {$(\Rmin,K)$-balanced} if 
\begin{align*} 
  L\leq K\quad \mbox{ and }\quad \Rmin \leq \min_{l\OneTo{L}} R_l  \ . 
\end{align*}
An optimal $(\Rmin,K)$-balanced solution has smallest cost among all solutions with $(\Rmin,K)$-balanced membership values. 
An optimal $(0,K)$-balanced solution is an optimal solution to the fuzzy $K$-means problem. 
\end{definition}

\begin{theorem}\label{thm:sfm:rand}
	There is a randomized algorithm that, given $X=\{(x_n,w_n)\}_{n\OneTo{N}}\subset \R^D\times \Q_{\geq 0}$, $K\in\N$, $m\in\N$, $\epsilon\in(0,1]$, and $\alpha\in(0,1]$, 
	computes $(C,R)$ such that with constant probability
	\[ \phiXm\left(C,R\right) \leq \left(1+\epsilon\right) \phiXm\left(C^{opt}_{\Rmin,K},R^{opt}_{\Rmin,K}\right)\ , \]
	where $\left(C^{opt}_{\Rmin,K},R^{opt}_{\Rmin,K}\right)$ is an {optimal} $(\Rmin,K)$-balanced solution with 
	\[ \Rmin =  \alpha \sum_{n=1}^N w_n\ .\] 
	The algorithms' runtime is bounded by 
	$N \cdot D \cdot 2^{\mathcal{O}\left( K^2\cdot \frac{1}{\epsilon}\log\left(\frac{1}{\alpha\epsilon}  \right)\right)}$.
\end{theorem} 
We can boost the probability of success to an arbitrary $1-\delta$ by simply repeating the algorithm $\log(1/\delta)$ times. 
Observe that the running time basically coincides with the running time of an algorithm that applies the superset sampling technique to the $K$-means problem \cite{ackermann10}.

The restriction $\Rmin = \alpha \sum_{n=1}^N w_n$ can also be seen as a restriction on the deviation between the single cluster weights $R_k$ and the average cluster weight $R_{avg}=1/K\cdot \sum_{k=1}^K R_k$. 
More precisely, using the Cauchy-Schwarz inequality, 
\begin{align*} 
   \frac{1}{K^{m-1}} \cdot \sum_{n=1}^N w_n
   \leq K\cdot R_{avg}   
   = \sum_{n=1}^N \left(\sum_{k=1}^K \rnk^m \right) w_n
   \leq \sum_{n=1}^N w_n\ . 
\end{align*}
Hence, if $\Rmin \geq \alpha \sum_{n=1}^N w_n$, then for all $k\OneTo{K}$ we have $R_k/R_{avg}\in[\alpha K, K^m]$.
For example, by choosing $\alpha=2^{-\mathcal{O}(K)}$ we allow a deviation of a factor $2^{-\mathcal{O}(K)}$ and obtain a runtime that is still linear in $N$ and becomes exponential in $K^3$.

The main idea behind this algorithm is the same as behind the PTAS described in Theorem~\ref{thm:sfm:det}. 
Knowing that for given fuzzy clusters there exist a hard clusters with similar characteristics, we apply a sampling technique known from the $K$-means problem. 
However, instead of combining this technique with exhaustive search, we directly apply the sampling technique to obtain our randomized algorithm from Theorem~\ref{thm:sfm:rand}. 
In other words, the PTAS from Theorem~\ref{thm:sfm:det} can be seen as a de-randomized version of this algorithm. 

%% file: techniques.tex
\section{Our Main Techniques}\label{sec:techniques}

\input{techniques_notation.tex}

\input{techniques_structure.tex}
\input{techniques_sampling.tex}

%% file: techniques_notation.tex
In this section, we describe the techniques that we use to prove Theorems~\ref{thm:sfm:det}, \ref{thm:sfm:rand}, and \ref{thm:candidate-means-algo}. 
To this end, we use the following notation.
\begin{definition}[Induced Solution]
  Let $X\subset\R^D\times \R$. 
  Membership values $R$ induce the solution $(\tilde C,R)$ where $\tilde C$ contains  the corresponding optimal mean vectors (cf. Equation~\eqref{eq:opt-means}). 
  Mean vectors $C$ induce the solution $(C,\tilde R)$ where $\tilde R$ contains the corresponding optimal membership values (cf. Equation~\eqref{eq:opt-membership}). 
  We denote the costs of the induced solutions by $\phiXm(R)$ and $\phiXm(C)$, respectively. 
\end{definition}
Observe that for all means $C=\muKSet$ and memberships $R=\rNKSet$ we have
 \begin{align} 
\phiXm(C) \leq \phiXm(C,R) \quad\mbox{and}\quad  \phiXm(R) \leq \phiXm(C,R)\ . \label{eq:preliminary:only_rnk_or_mu}
\end{align}

%% file: techniques_structure.tex
\subsection{Structure of the Fuzzy \texorpdfstring{$K$}{K}-Means Problem}\label{sec:techniques:structure}

There are two aspects of the structure of the fuzzy $K$-means problem that we exploit extensively. 
First, there is a coarse but still useful relation between the $K$-means and the fuzzy $K$-means cost function. 
Recall that optimal solutions of the fuzzy $K$-means problem seem to have a substantially different structure than optimal solutions of the $K$-means problem (cf. Section~\ref{sec:contrib:complexity}). 
Nonetheless, the fuzzy $K$-means and $K$-means cost of solutions induced the same set of mean vectors differ by at most a factor of $K^{m-1}$.
We use this result when transfering the ideas behind the coreset construction of \cite{HarPeled03} in order to obtain a candidate set of means. 
\begin{definition}[$K$-means]
  For $X=\{(x_n,w_n)\}_{n\OneTo{N}}\subset \R^D\times \R_{\geq 0}$ and $C=\muKSet\subset\R^D$ we define $\km_X\left(C\right) \coloneqq\sum_{n=1}^N w_n \min_{k\OneTo{K}} \norm{x_n-\mu_k}^2$. 
\end{definition}
\begin{lemma}\label{lem:kmeans}
  Let $X\subset\R^D\times \R_{\geq 0}$, $m\in\N$, and $C\subset\R^D$ with $\abs{C}=K$. 
  Then, 
  \[ \frac{1}{K^{m-1}} \km_X(C) \leq \phiXm(C) \leq \km_X(C)\ .\]	
\end{lemma}
\begin{proof}
	Obviously, $\phi_X^{(m)}(C) \leq \km_X(C)$.
        Let $X=\{(x_n,w_n)\}_{n\OneTo{N}}$. 
        Let $\rnkSet$ be the optimal memberships induced by $C=\muKSet$. 
	Using the Cauchy-Schwarz inequality,
	$\frac{1}{K^{m-1}} \cdot \km_X(C) \leq  \sum_{n=1}^N \left(\sum_{k=1}^K r_{nk}^m \right) w_n \left(\min_{k\OneTo{K}} \norm{x_n - \mu_k}^2\right) \leq \phi_X^{(m)}(C)$. 
\end{proof}

Second, we can ignore fuzzy clusters with too small a weight. 
For each optimal fuzzy $K$-means solution, there exists a fuzzy $L$-means solution with $L\leq K$ clusters such that each cluster has a certain minimum weight $\Rmin$ while the cost are only at most a factor $(1+\epsilon)$ worse than the cost of the optimal solution.
Recall that we denote such clusterings as $(\Rmin,K)$-balanced (cf. Definition~\ref{def:intro:r-balanced}). 
More precisely, $\Rmin$ only depends on $\epsilon$, $m$, $K$, and the smallest weight of a point in $X$.  
This result becomes important when we apply sampling techniques. 
When sampling points from $X$, we can only expect to sample points from a certain cluster if this cluster is large enough.

\begin{lemma}\label{lem:existence-of-Rmin-balanced-approx}
	Let $X=\{(x_n,w_n)\}_{n\OneTo{N}}\subset\R^D\times \R$, $m\in\N$, $K\in\N$, and $\epsilon\in[0,1]$.
	
	There exist $(\Rmin,K)$-balanced membership values $R$ such that
	 \[ \phiXm(R) \leq (1+\epsilon) \phiopt \] 	
	 where 
	 \[ \Rmin =  \left(\frac{\epsilon}{4mK^2}\right)^m \cdot \min_{n\OneTo{N}} w_n \ .\]
\end{lemma}
\begin{proof}
  Consider an arbitrary but fixed solution $\muKSet$. 
  Let $\rnkSet$ be the optimal membership values induced by $\muKSet$. 
  Assume that for some $l\OneTo{K}$ we have $R_l =\sum_{n=1}^N \rnl^m w_n \leq \left(\frac{\epsilon}{4mK^2}\right)^m \cdot \min_{n\OneTo{N}}w_n$. 
  Thus, we have $r_{nl} \leq \frac{\epsilon}{4mK^2}$ for all $n\OneTo{N}$.
  
  Consider an arbitrary $n\OneTo{N}$. 
  Since $\sum_{k=1}^K \rnk =1$ and $\rnl\leq \frac{1}{K}$, there exists some $k(n)\OneTo{K}$ with $k(n)\neq l$ such that $r_{nk(n)} \geq \frac{1}{K}$. 
  Since $r_{nl} \leq \frac{\epsilon}{4mK^2}$, we have $r_{nl} \leq \frac{\epsilon}{4mK}r_{nk(n)}$. 
  Hence, 
  \begin{align} 
  (r_{nk(n)}+r_{nl})^{m} 
  \leq \left( 1+\frac{\epsilon}{4mK}\right)^m r_{nk(n)}^m
  \leq \left(1+\frac{\epsilon}{2K}\right)\cdot r_{nk(n)}^m\ .\label{eq:kn-member}
  \end{align}
  Due to the optimality of the $r_{nl}$ and $r_{nk(n)}$ (cf. Equation~\eqref{eq:opt-membership}) and since $r_{nl}\leq r_{nk(n)}$, we have
  \begin{align}
    \norm{x_n-\mu_{k(n)}}^2 \leq \norm{x_n-\mu_l}^2\ .\label{eq:kn-dist}
  \end{align}
  Hence, 
  \begin{align*}
   &\phiXm(\muKSet\setminus \{\mu_l\}) \\
   &\leq \sum_{n=1}^N \sum_{\substack{k\OneTo{K}\\ k\neq l,k(n)}} \rnk^{m} w_n\norm{x_n-\mu_k}^2 
   + \sum_{n=1}^N  (r_{nk(n)}+r_{nl})^{m} w_n \norm{x_n-\mu_{k(n)}}^2 \\ 
   &\leq \sum_{n=1}^N \sum_{\substack{k\OneTo{K}\\ k\neq l,k(n)}} \rnk^{m}w_n\norm{x_n-\mu_k}^2 + 
   \sum_{n=1}^N  \left(1+\frac{\epsilon}{2K}\right) r_{nk(n)}^m w_n\norm{x_n-\mu_{k(n)}}^2\tag{by Eq.~\eqref{eq:kn-member}}\\
   &\leq  \left(1+\frac{\epsilon}{2K}\right) \phiXm(\muSet) \tag{by Eq.~\eqref{eq:kn-dist}}\ .
  \end{align*}
  
  Now consider an optimal solution $\muKSet$. 
  Let $B \subset \muKSet$ be the set containing the means $\mu_l$ where $R_l\leq \left(\frac{\epsilon}{4mK^2}\right)^m \cdot \min_{n\OneTo{N}} w_n$. 
  Note that $B\leq K-1$. 
  Let $L = K-\abs{B}$. 
  Then, by the above there exists a set of membership values $\tilde r_L=\{\tilde r_{nl}\}_{n\OneTo{N}, l\OneTo{L}}$ such that
  \[ \phiXm(\muKSet\setminus B,\ \tilde r_L)\leq\left(1+\frac{\epsilon}{2K}\right)^K\phiopt
  \leq\left(1+\epsilon\right)\phiopt\   \]
  and $\sum_{n=1}^N (\tilde r_{nl})^m w_n \geq \left(\frac{\epsilon}{4mK^2}\right)^m \cdot \min_{n\OneTo{N}} w_n$ for all $l\OneTo{L}$. 
  Finally, observe that $\phiXm(\tilde r_L)\leq \phiXm(\muKSet\setminus B, \tilde r_L)$.
\end{proof}

%% file: techniques_sampling.tex
\subsection{Sampling Techniques}\label{sec:techniques:sampling}

Our algorithms from Theorems~\ref{thm:sfm:det} and \ref{thm:sfm:rand} both use sampling techniques. 
First, we show that there exist hard clusters suitably imitating soft clusters. 
Second, we show how to construct a candidate set that contains good approximations of the means of these (unknown) hard clusters.

\subsubsection{Relating Fuzzy to Hard Clusters}\label{sec:techniques:sampling:soft-to-hard}

A fundamental result is that, given fuzzy clusters with not too small a weight, there always exist hard clusters that exhibit characteristics similar to those of the fuzzy clusters. 

\begin{definition}[Cost of a Fuzzy Cluster, Hard Clusters]\label{def:partials}
Let $X = \{(x_n,w_n)\}_{n\OneTo{N}}\subset\R^D\times \R$ and $K\in\N$. 
Given memberships $\rNKSet$ and induced means $\muKSet$, we let
\[  \phiXmk(\rnSet) \coloneqq \sum_{n=1}^N \rnk^m w_n \norm{x_n - \mu_k}^2\enspace .\]
for all $k\OneTo{K}$. For all hard clusters $C\subset X$, $C\neq\emptyset$, we define 
\begin{align*} 
  &w(C)\coloneqq \sum_{(w_n,x_n)\in C} w_n,
  \quad \mu(C)\coloneqq \frac{\sum_{(w_n,x_n)\in C} w_n\cdot x_n}{w(C)} 
  \quad \mbox{ and }   \\
 &\km(C)\coloneqq \sum_{(w_n,x_n)\in C} w_n \norm{x_n-\mu(C)}^2  \enspace .
\end{align*}
\end{definition}

\begin{theorem}[Existence of Similar Hard Clusters]\label{thm:hard-clustering}
	Let $X=\{(x_n,w_n)\}_{n\OneTo{n}}\subset \R^D\times \R_{\geq 0}$ and $\epsilon\in(0,1]$. 
	Let $\rnkSet$ be memberships values, and let $\muSet$ be the corresponding optimal mean vectors.
	
	If $\min_{k\OneTo{K}} R_k   \geq {16K \wmax}/{\epsilon}$, where $\wmax = \max_{n\OneTo{N}} w_n$, 
	then there exist pairwise disjoint sets $C_1,\ldots,C_K \subseteq X$ such that for all $k\OneTo{K}$
	\begin{align}
		 w(C_k)
		& \geq \frac{1}{2}R_k \label{eq:proof:r}\enspace , \\
		 \norm{\mu(C_k) - \mu_k}^2 
		& \leq   \frac{\epsilon}{2 R_k}  \cdot  \phiXmk\left(\rnSet\right) \label{eq:proof:mu} \enspace \mbox{ and}\\
		 \km(C_k) 
		& \leq 4K \cdot \phiXmk\left(\rnSet\right) \label{eq:proof:sigma} \enspace .
	\end{align}
\end{theorem}
\begin{proof}[Idea of the Proof in Section~\ref{sec:stochastic-fuzzy-k-means}]
To prove this theorem, we apply the probabilistic method.
Consider a random process that samples a hard assignment for each $(x_n,w_n)\in X$ independently at random by assigning $(x_n,w_n)\in X$ to the $k^{th}$ cluster with probability $\rnk^m$.
This assignment can be considered a binary random variable $\znk\in\{0,1\}$ with expected value $E[\znk] = \rnk^m$.
We compare the resulting hard clusters $C_k = \{(x_n,w_n)\in X\ |\ \znk = 1\}$ with the fuzzy clusters $C_k^F = \{(x_n,w_n\cdot \rnk^m)\in X\ |\ x_n\in X\}$.
With positive probability, the weights, means, and costs of the constructed hard clusters satisfy the given properties. 
\end{proof}

Unfortunately, to the best of our knowledge, the hard clusters $C_k$ \emph{do not exhibit any structure}, e.g. are not necessarily convex and do not necessarily cover $X$.
In the next section, we describe how the superset sampling technique \cite{inaba94} \cite{kumar04} can be used approximate the means $\mu(C_k)$  well. 
It is not clear, how other techniques which do not solely rely on sampling can be applied. 
For instance, we presume that the sample and prune technique from \cite{ackermann10} and the $K$-means++ algorithm \cite{arthur07} require that the convex hulls of clusters do not overlap.
The hard clusters $C_k$ whose existence we can prove do not necessarily have this property.

\subsubsection{Superset Sampling}\label{sec:techniques:sampling:super}

From Theorem~\ref{thm:hard-clustering}, we know that there exist hard clusters similar to given fuzzy clusters. 
Hence, means that are sufficiently close to the means of the hard clusters induce a solution that is also close to the solution given by the fuzzy clusters. 
The superset sampling technique introduced by \cite{inaba94} \cite{kumar04} can be used to find such means. 
More precisely, we can construct a candidate set containing good approximations of the means of unknown hard clusters if these hard clusters do not have too small a weight compared to the weight of the give point set. 
\begin{theorem}\label{thm:superset-sampling}
	There is a randomized algorithm that, given $X=\{(x_n,w_n)\}_{n\OneTo{N}} \subset \R^D\times \Q_{\geq 0}$, $K\in \N$, $\epsilon\in(0,1]$, and $\alpha\in(0,1]$, constructs a set $T\subset(\R^D)^K$ in time 
	\[   \mathcal{O} \left( D \cdot \left(N 
	+ 2^{\frac{K}{\epsilon}\cdot\log \left(\frac{1}{\alpha\epsilon}\right)\cdot \log\left(\log(K)\right)} \right)\right) \enspace \]
	such that for an arbitrary but fixed set $\{ C_k\}_{k\OneTo{K}}$ of (unknown) sets $C_k\subseteq X$, with constant probability, there exists a $\left( \tmu_k \right)_{k\OneTo{K}}\in T$  such that for all $k\OneTo{K}$ where
	\[ w(C_k) \geq \alpha\cdot \sum_{n=1}^N w_n \] 
	we have
	\[ \norm{\tmu_k - \mu( C_k)}^2 \leq \frac{\epsilon}{w(C_k)} \km(C_k) \ .\]
\end{theorem}

We apply this result in two different ways:
Firstly, we apply the superset sampling technique directly to obtain a {randomized approximation algorithm}. 
That is, we generate the candidate set $T$ and determine the candidate with the smallest fuzzy $K$-means costs. 
Secondly, we use exhaustive search to obtain a {deterministic approximation algorithm}.
That is, we generate all candidates that the algorithm from Theorem~\ref{thm:superset-sampling} might possibly generate and choose the best of these candidates.  
Note that the latter approach does not require that the weights of the fuzzy clusters make up a certain fraction of the weight of the point set.

%% file: algorithm.tex
\section{Proof Sketches}\label{sec:sketch}

\input{algorithm_sfm.tex}

\input{algorithm_mean-candidates.tex}

%% file: algorithm_sfm.tex
\subsection{Relating Fuzzy to Hard Clusters (Theorems~\ref{thm:sfm:det} and \ref{thm:sfm:rand})}\label{sec:sketch:sfm}

The following proposition is the basis for the proofs of Theorems~\ref{thm:sfm:det} and \ref{thm:sfm:rand}. 

\begin{proposition}\label{prop:sfm} 
	There is a randomized algorithm that, given $X=\{(x_n,w_n)\}_{n\OneTo{N}}\subset \R^D\times \Q_{\geq 0}$, $K\in N$, $m\in\N$, $\epsilon\in(0,1]$, and $\alpha\in(0,1]$, computes mean vectors $C\subset \R^D$, $\abs{C}=K$, 
	such that with probability at least $1/2$ we have
	\[ \phiXm\left( C \right) \leq \left(1+\epsilon\right) \phiXm\left( R^{opt}_{\Rmin,K}\right)\ ,  \]
	where the memberships $R^{opt}_{\Rmin,K}$ induce an optimal $(\Rmin,K)$-balanced solution with 
	\[ \Rmin = \max\left\{ \alpha \cdot \sum_{n=1}^N w_n ,\ \frac{16K\wmax}{\epsilon}\right\}\quad \mbox{ and }\quad \wmax = \max_{n\OneTo{N}} w_n\ . \] 
	The algorithms' runtime is bounded by $N D \cdot 2^{\mathcal{O}\left( K^2/\epsilon\cdot\log\left(1/(\alpha\epsilon)\right)\right)}$. 
	
	There is a deterministic version of this algorithm that, given $X=\{(x_n,w_n)\}_{n\OneTo{N}}\subset \R^D\times \Q_{\geq 0}$, $K\in N$, $m\in\N$, and $\epsilon\in(0,1]$, computes mean vectors which induce an $(1+\epsilon)$-approximation to an optimal $\left(\Rmin,\ K\right)$-balanced solution
	where
	$\Rmin = \frac{16K\wmax}{\epsilon}$.
	The runtime of this algorithm is bounded by $D \cdot N^{\mathcal{O}\left( K^2/\epsilon\right)}$.	
\end{proposition} 
\begin{proof}
    First, we prove that there is a randomized algorithm as described in the theorem. 
    Fix an optimal $(\Rmin,K)$-balanced solution with memberships $\rNLSet$. 
    Let $\{\mu_l\}_{l\OneTo{L}}$ be the optimal mean vectors induced by the fixed $\rNLSet$.
    
    By  Theorem~\ref{thm:hard-clustering}, we know that there exist hard clusters $\{C_l\}_{l\OneTo{L}}\subseteq X$ similar to the fuzzy clusters given by the $\rNLSet$. 
    Due to Equation~\eqref{eq:proof:r} and since  $R_l = \sum_{n=1}^N \rnl^m w_n \geq \alpha\cdot \sum_{n=1}^N w_n$, we have $w(C_l) \geq {R_l}/{2} \geq \alpha/2\cdot \sum_{n=1}^N w_n$. 
	
    Now, apply Theorem~\ref{thm:superset-sampling}. 
    Consider the set $T$ that is generated as in Theorem~\ref{thm:superset-sampling} with 	$\epsilon$ set to ${\epsilon}/{(16K)}$ and $\alpha$ set to ${\alpha}/{2}$. 
    With probability at least $1/2$, the set $T$ contains a candidate $(\tmu_k)_{k\OneTo{K}}$ such that for all $l\OneTo{L}$ we have
	\begin{align}
	    \norm{\tilde\mu_l - \mu(C_l)}^2 
	  \leq \frac{\epsilon}{16K w(C_l)}\km(C_l)  \enspace . \label{eq:proof:dist-of-means}
	\end{align}

    We can upper bound the cost of the means $\{\tmu_k\}_{k\OneTo{K}}$ by
	\begin{align*}
		\phiXm(\{\tmu_k\}_{k\OneTo{K}}) 
		&\leq \phiXm(\{\tmu_l\}_{l\OneTo{L}}) \tag{since $(\{\tmu_l\}_{l\OneTo{L}}\subseteq \{\tmu_k\}_{k\OneTo{K}}$} \\
		&\leq \phiXm(\{\tmu_l\}_{l\OneTo{L}}, \rNLSet)  \\
		&\leq \phiXm(\rNLSet)
		+  \sum_{l=1}^L R_l \norm{\mu_l - \tmu_l}^2  \\
		&\leq \phiXm(\rNLSet)
		+ 2 \sum_{l=1}^L R_l \left( \norm{\mu_l -\mu(C_l)}^2 + \norm{\mu(C_l) - \tmu_l}^2 \right) \ ,
	\end{align*}
	where the second to last inequality is well-known (a proof can be found in Section~\ref{sec:stuff} (Lem.~\ref{lem:magic-formula})) and where the last inequality is due to the fact that $\forall a,b\in\R: (a+b)^2\leq2a^2+2b^2$.

	Due to Equation~\eqref{eq:proof:mu}, we obtain
	\begin{align*} 
	  \sum_{l=1}^L R_l \norm{\mu_l -\mu(C_l)}^2
	 &\leq \sum_{l=1}^L  R_l \cdot \frac{\epsilon}{2R_l} \cdot \phiXml(\rnSet)  
	 =  \frac{\epsilon}{2} \cdot \phiXm \left(\rNLSet\right)\ . 
	\end{align*}
	Furthermore, we have
	\begin{align*} 
	 \sum_{l=1}^L R_l \norm{\mu(C_l) - \tmu_l}^2 
	&\leq \sum_{l=1}^L  R_l \cdot \frac{\epsilon}{16Kw(C_l)}\cdot \km(C_l) \tag{by Equation~\eqref{eq:proof:dist-of-means}}\\
	&\leq \frac{\epsilon}{2} \cdot \phiXm(\rNLSet) \tag{by Equations~\eqref{eq:proof:r}, \eqref{eq:proof:sigma}}\ .
	\end{align*}
	Therefore, with probability $1/2$, there exists a candidate tuple $(\tmu_k)_{k\OneTo{K}}\in T$ that induces the desired approximate solution of the given optimal $(\Rmin,K)$-balanced solution $\rNLSet$.  
	Hence, the candidate with the smallest fuzzy $K$-means cost among all the possible candidates 
	satisfies the approximation bound as well.
	This concludes the proof of the existence of the randomized algorithm.
	
	To obtain a deterministic algorithm, we use exhaustive search. 
	Using the same argument as above, but setting $\alpha$ to ${\min_{l\OneTo{L}}R_l}/{\sum_{n=1}^N w_n}$, one can obtain the desired approximation with positive probability. 
	The set $T$ contains tuples of means of multi-sets, which are subsets of $X$ with size ${32K}/{\epsilon}$ (cf. Section~\ref{sec:superset-sample}). 
	Hence, by testing all combinations of means of \emph{all possible} multi-sets,  which are subsets of $X$ with size ${32K}/{\epsilon}$, we obtain a candidate set containing the desired approximation.
	Note that there are at most $N^{32K/\epsilon}$ such subsets. 
	Again, the candidate with the smallest cost yields the desired approximation as well. 
\end{proof}

\begin{proof}[Proof of Theorem~\ref{thm:sfm:det}]
 Let $\rNLSet$ be the  memberships from Lemma~\ref{lem:existence-of-Rmin-balanced-approx} with $\epsilon$ replaced by $\epsilon/4$. 
 Then, we have $\phiXm(\rNLSet)\leq (1+\epsilon/4)\phiopt$. 
 
 Let $X'$ be the point set containing $c = \left\lceil 2\cdot 16^{m+1} m^m K^{2m+1}  \wmax / (\wmin \epsilon^{m+1}) \right\rceil$ copies of each point. 
 From Lemma~\ref{lem:existence-of-Rmin-balanced-approx} we know that for all $l\OneTo{L}$ we have $\sum_{n=1}^N \rnl^m w_n \geq \left({\epsilon}/({16mK^2})\right)^m \cdot \wmin $.
 Hence, for all $l\OneTo{L}$ we have $\sum_{x_n\in X'} \rnl^m w_n 
  = c\sum_{x_n\in X} \rnl^m w_n 
  \geq c \left({\epsilon}/({16mK^2})\right)^m  \wmin 
  \geq {32 K \wmax}/{\epsilon}$. 
 Thus, we can apply Proposition~\ref{prop:sfm} (with $X$ replaced by $X'$, $\epsilon$ replaced by $\epsilon/2$, and $\rNLSet$ replaced by a set containing $c$ copies of the memberships in $\rNLSet$) and obtain means $\tmuKSet$. 
 Observe that for all $C\subset\R^D$ we have $\phim_{X'}(C) = c \cdot  \phiXm(C)$. 
 Thus, $\phiXm(\tmuKSet)\leq (1+\epsilon/2)\phiXm(\rNLSet) \leq (1+\epsilon)\phiopt$. 
 
 Since we apply the algorithm from Lemma~\ref{lem:existence-of-Rmin-balanced-approx} to a set containing $c$ copies of the points in $X$, its runtime is bounded by $D \cdot (c\cdot N)^{\mathcal{O}\left( \frac{K^2}{\epsilon}\right)}$. 
 Finally, note that
 $(c\cdot N)^{\mathcal{O}\left( \frac{K^2}{\epsilon}\right)}
 \subseteq 
 N^{\mathcal{O}\left(  \log(c) \cdot \frac{K^2}{\epsilon} \right)}
 \subseteq 
 N^{\mathcal{O}\left( \left( m\log(K)+m\log(m)+m\log\left(\frac{1}{\epsilon}\right)  + \log\left(\frac{\wmax}{\wmin}\right)\right) \cdot \frac{K^2}{\epsilon} \right)}
 \subseteq 
 N^{\mathcal{O}\left( K^2 \log(K) \cdot \frac{1}{\epsilon}\left(m \log\left(\frac{m}{\epsilon}\right) +   \log\left(\frac{\wmax}{\wmin}\right)\right) \right)}$.
\end{proof}

\begin{proof}[Proof of Theorem~\ref{thm:sfm:rand}]
 Construct a point set $X'$ that contains $c = \left\lceil {16 K}/({\alpha\epsilon})\right\rceil$ copies of each point in $X$. 
 Fix an arbitrary $k\OneTo{K}$. 
 Since $\Rmin \geq \alpha \sum_{n=1}^N w_n$, we have $\sum_{x_n\in X} \rnk^m w_n \geq \alpha \sum_{x_n\in X} w_n$. 
 By definition of $X'$, $\sum_{x_n\in X'}\rnk^m w_n =  \sum_{x_n\in X} c \cdot w_n \rnk^m \geq  c \cdot \alpha  \sum_{x_n\in X} w_n = \alpha \sum_{x_n\in X'} w_n$.
 Also note that $c \cdot \alpha  \sum_{x_n\in X} w_n \geq c \cdot \alpha \cdot \wmax \geq {16K\wmax}/{\epsilon}$. 
 Hence, we can apply Proposition~\ref{prop:sfm} to $X'$ to find $\tmuSet$ approximating the $\muSet$ induced by (copies of) the memberships in $\rnkSet$, with constant probability. 
 Observe that for all $C\subset\R^D$, $\phim_{X'}(C) = c \cdot  \phiXm(C)$. 
 Hence, $\phiXm(\tmuSet)\leq (1+\epsilon)\phiopt$. 
 
 Since we apply the algorithm from Lemma~\ref{lem:existence-of-Rmin-balanced-approx} to a set containing $c$ copies of the points in $X$, it requires runtime $(c\cdot N) \cdot D \cdot 2^{\mathcal{O}\left( K^2/\epsilon\cdot\log\left(1/(\alpha\epsilon)\right)\right)}$. 
 Finally, note that
 $c 
 = 
 2^{\mathcal{O}\left( \log(c)\right)}
 \subseteq 
 2^{\mathcal{O}\left( \log(K) + \log(1/(\alpha\epsilon)) \right) }
 \subseteq 
 2^{\mathcal{O}\left( K^2/\epsilon\cdot \log\left(1/(\alpha\epsilon) \right)\right)}$, assuming that $\alpha \epsilon\leq 1/2$. 
\end{proof}

%% file: algorithm_mean-candidates.tex
\subsection{Candidate Set Search for Mean Vectors (Proof of Theorem~\ref{thm:candidate-means-algo})}\label{sec:sketch:candidate}

Using ideas behind the coreset construction of \cite{HarPeled03}, we can construct a candidate set of mean vectors. 
The algorithm that creates and tests all these candidates and finally chooses the best candidates satisfies the properties from Theorem~\ref{thm:candidate-means-algo}.

	\begin{theorem}[Candidate Set]\label{thm:candidate-means}
		Let $X\subset\R^D$, $K\in\N$, and $\epsilon\in(0,1]$.
		
		There exists a set $\cG\subset\R^D$ with size
		\[ \abs{\cG}=\cO\left(K^{mD+1}\epsilon^{-D}m\log\left(\frac{mK}{\epsilon}\right)\log(N)\right) \]
		that contains $\muKSet\subset\cG$ with 
		\[ \phiXm(\muKSet) \leq (1+\epsilon)\phiopt\ . \]
		The set $\cG$ can be computed in time $\cO(N(\log(N))^K\epsilon^{-2K^2D} + NKD\abs{\cG})$.
	\end{theorem}
	\begin{proof}[Sketch of Proof in Section~\ref{sec:proof:candidates}]
	The idea behind the coreset construction of \cite{HarPeled03} can be used to construct a candidate set of mean vectors. 
	Part of the construction is a constant factor approximation of the $K$-means problem. 
	To this end, we use the deterministic algorithm presented in \cite{Matousek00}, which requires time $\cO\left(N(\log(N))^K\epsilon^{-2K^2D}\right)$. 
	\end{proof}

%% file: conclusion.tex
\section{Future Work \& Open Problems}

A goal of further research is to examine whether Theorem~\ref{thm:hard-clustering} can be transferred to other soft clustering problems. 
In particular, we hope to obtain a constant factor approximation algorithm for the maximum likelihood estimation problem for mixtures of Gaussian distributions, e.g. by using the results from \cite{bloemer13}.

It is an open question whether we are able to classify hardness of approximation of fuzzy $K$-means. 
We conjecture that, just as the classical $K$-means problem, if P$\neq$NP, then there is no PTAS for the fuzzy $K$-means problem for arbitrary $K$ and $D$.

%% file: appendix.tex
\section{Full Proofs}\label{sec:appendix}

\input{app_stuff.tex}

\input{app_stochastic-fuzzy-k-means.tex}

\input{app_superset-sampling.tex}

\input{app_candidate-means.tex}

\input{app_solv.tex}
\input{app_bad-example.tex}

%% file: app_stuff.tex
\subsection{Preliminaries}\label{sec:stuff}

In this section we introduce some notation and lemmata that are used throughout the rest of this appendix.

\subsubsection{Hard Clusters and $K$-Means Costs}

The following definition restates some of the notation already presented in Definition~\ref{def:partials}.

\begin{definition}[Hard Clusters]\label{def:hard-clusters}
For weighted point sets  $C\subset \R^D\times \R_{\geq 0}$ , we let
\begin{align*} 
  w(C) &\coloneqq \sum_{(x_n,w_n)\in C} w_n \enspace , \\
  \mu(C) &\coloneqq \frac{\sum_{(x_n,w_n)\in C}w_n x_n}{w(C)}\  \mbox{ and }\\
  \km(C) &\coloneqq \sum_{(x_n,w_n)\in C} w_n \norm{x_n - \mu(C)}^2 \enspace .
\end{align*}
\end{definition}

The following lemma is well known (e.g. used in the proof of Theorem~2 in \cite{inaba94}).
\begin{lemma}\label{lem:magic-formula}
 Let $C\subset \R^D\times \R$ be a weighted point set and $\mu\in\R^D$.
 Then,
 \[ \sum_{(x_n,w_n)\in C} w_n \norm{x_n -\mu}^2 
 = \km(C)
 + w(C) \norm{\mu - \mu(C)}^2\enspace .  \]
\end{lemma}
\begin{proof}
 \begin{align*}
    &\sum_{(x_n,w_n)\in C} w_n\norm{x_n-\mu}^2 \\
    &=\sum_{(x_n,w_n)\in C} w_n\norm{x_n-\mu(C)+\mu(C)-\mu}^2 \\
    &=\sum_{(x_n,w_n)\in C} w_n\sca{x_n-\mu(C)+\mu(C)-\mu}{x_n-\mu(C)+\mu(C)-\mu} \tag{scalar product}\\ 
    &=\sum_{(x_n,w_n)\in C} w_n \left(\sca{x_n-\mu(C)}{x_n-\mu(C)}+2\sca{x_n-\mu(C)}{\mu(C)-\mu}+\sca{\mu(C)-\mu}{\mu(C)-\mu}\right)\tag{bilinearity and symmetry}\\
    &=\sum_{(x_n,w_n)\in C} w_n \left(\norm{x_n-\mu(C)}^2+2\sca{x_n-\mu(C)}{\mu(C)-\mu}+\norm{\mu(C)-\mu}^2\right)\\
    &=\sum_{(x_n,w_n)\in C} w_n\norm{x_n-\mu(C)}^2 
    + \sum_{(x_n,w_n)\in C} w_n \sca{x_n-\mu(C)}{\mu(C)-\mu}
 + w(C) \norm{\mu(C) - \mu}^2
 \end{align*}
 where due to the bilinarity of the scalar product and by the definition of $\mu(C)$
 \begin{align*}
    \sum_{(x_n,w_n)\in C} w_n \sca{x_n-\mu(C)}{\mu(C)-\mu}
    &=  \sca{\sum_{(x_n,w_n)\in C}w_n( x_n-\mu(C))}{\mu(C)-\mu}\\
    &=  \sca{\left(\sum_{(x_n,w_n)\in C} w_n x_n\right)- w(C)\cdot\mu(C)}{\mu(C)-\mu}\\
    &=  \sca{\vec 0}{\mu(C)-\mu} = 0
 \end{align*}
 which yields the claim.
\end{proof}

Furthermore, the cluster costs $\km(C)$ can be expressed via pairwise distances.
\begin{lemma}\label{lem:opt-kmeans-cost-via-distances-between-points}
  Let $C\subset \R^D\times \R$ be a weighted point set.
 Then,
 \[ \km(C)
 = \frac{1}{2\sum_{(x_n,w_n)\in C} w_n}\sum_{(x_n,w_n)\in C} \sum_{(x_l,w_l)\in C}  w_n w_l\norm{x_n-x_l}^2 \enspace .  \]
\end{lemma}

\begin{definition}[$K$-Means Costs (unweighted)]\label{def:kmeans}
For unweighted point sets $X\subset\R^D$, $x\in\R^D$, and finite sets $M\subset\R^D$ we let 
\begin{align*}
 d(x,M) &\coloneqq \min_{m\in M}\norm{x-m}\ ,\\
 \km_X(M) &\coloneqq \sum_{x\in X} d(x,M)^2 = \sum_{x\in X} \min_{m\in M}\norm{x-m}^2\ \mbox{ and }\\
 \km_{X,K} &\coloneqq \min_{\subalign{M&\subset R^D\\ \abs{M}&=K}} \km_X(M)\ .
\end{align*}
\end{definition}

\begin{definition}[Induced Partition (unweighted)]\label{def:induced-kmeans-partition}
 We say $\{C_k\}_{k=1}^K$ is a partition of $X$ induced by $C\subset\R^D$ if
 $C_k \subseteq \{ x_n\in X\; |\; \forall l\neq k: \norm{x_n-\mu_k}\leq \norm{x_n-\mu_l}\}$ 
 and $X=\dot\cup_{k=1}^K C_k$. 
\end{definition}

\subsubsection{Some Useful Technical Lemmas}

Besides, we make extensive use of the following simple technical lemmata.

\begin{lemma}\label{lem:abc}
 For all $a,b,c\in R^D$ we have
 \[  \norm{c-a}^2 -  \norm{c-b}^2 \leq \norm{a-b}^2 + 2 \norm{a-b} \norm{c-b} \]
\end{lemma}
\begin{proof}
\begin{align*}
   \norm{a-c}^2 -  \norm{b-c}^2
   &\leq \left\vert \norm{a-c}^2 -  \norm{b-c}^2 \right\vert \\
   &= \left\vert \norm{a-b+b-c}^2 -  \norm{b-c}^2 \right\vert \\
    &= \left\vert  \norm{a-b}^2 + 2\sca{a-b}{b-c} \right\vert \\
   &\leq \norm{a-b}^2 + 2\left\vert \sca{a-b}{b-c}\right\vert \\
   &\leq \norm{a-b}^2 + 2\norm{a-b}\norm{b-c} \tag{Cauchy-Schwarz}
\end{align*}
\end{proof}

\begin{lemma}\label{lem:stuff:1+eps}
 Let $\epsilon\in[0,1]$, $c>1$, and $m\in\N$.
 Then, for all $i\OneTo{m}$ it holds
 \[ \left(1+\frac{\epsilon}{2 mc}\right)^i\leq 1+i\cdot \frac{\epsilon}{mc} \ . \]
\end{lemma}

\begin{lemma}\label{lem:squares}
	For all $a,b\in\R$ we have
	\begin{enumerate}
		\item $2ab \leq a^2+b^2$,
		\item $(a+b)^2 \leq 2(a^2+b^2)$ and
		\item $(a+b+c)^2 \leq 3(a^2+b^2+c^2)$.
	\end{enumerate}
\end{lemma}

%% file: app_stochastic-fuzzy-k-means.tex
\section{Stochastic Fuzzy Clustering (Proof of Theorem~\ref{thm:hard-clustering})}\label{sec:stochastic-fuzzy-k-means}

In this section, we first describe a random process that, given some fuzzy $K$-means clusters, creates $K$ hard clusters. 
We define different quantities that describe the different clusters and derive probabilistic bounds on the similarity between them with respect to these quantities.

\subsection{Setting and Random Process}\label{sec:stochastic-fuzzy-k-means:rand-process}
In the following we consider arbitrary but fixed memberships $\{\rnk\}_{n,k}$.
These membership values induce \emph{fuzzy clusters}.
We say that the $k^{th}$ fuzzy cluster has weight $R_k$, mean $\mu_k$, and cost $\phiXmk(\{\rnk\}_n)$, where the $\mu_k$ are the optimal means with respect to the given memberships. 
Recall that by Equation~\eqref{eq:opt-means}, Equation~\eqref{eq:Rk}, and Definition~\ref{def:partials} we have
\begin{align}
 R_k &= \sum_{n=1}^N r_{nk}^m w_n \ , \nonumber\\
 \mu_k  &=  \frac{\sum_{n=1}^N r_{nk}^m w_n x_n}{\sum_{n=1}^N \rnk^m w_n }\ , \mbox{ and } \nonumber\\
 \phiXmk(\{\rnk\}_n) &=  \sum_{n=1}^N r_{nk}^m w_n \norm{x_n - \mu_k}^2\ .\label{eq:fuzzy-quantities}
\end{align}

We consider the following random process that aims to imitate the fuzzy clustering.
Given the fixed memberships $\{\rnk\}_{n,k}$, the process samples an assignment for each $(x_n,w_n)\in X$ independently at random. 
Formally, we describe these assignments by random variables $(z_{nk})\in\{0,1\}^{K}$ with $\sum_{k=1}^K \znk \in\{0,1\}$. 
They are sampled according to the following distribution:
\begin{enumerate}
  \item With probability $r_{nk}^m$, the process assigns $x_n$ to the $k^{th}$ cluster. 
  That is, $\Pr(\znk=1) = \rnk^m$. 
  \item The process does not assign $x_n$ to any cluster at all with probability $1 - \sum_{k=1}^K r_{nk}^m$. 
    That is,  $\Pr(\forall k\OneTo{K}:\  \znk=0) = 1-\sum_{k=1}^K r_{nk}^m$. 
\end{enumerate}
This  process constructs \emph{hard clusters} $\{C_k\}_{k\OneTo{K}}$ with $C_k=\{(x_n,w_n)\in X|\znk=1\}\subseteq X$ that do not necessarily cover $X$. 
Using Definition~\ref{def:hard-clusters}, we can conclude that
\begin{align*}
  w(C_k) &=  \sum_{n=1}^N \znk w_n  \  ,  \\
 \mu(C_k) &=  \frac{\sum_{n=1}^N \znk w_n x_n}{w(C_k)} \ \mbox{, and } \\
 \km(C_k) &=  \sum_{n=1}^N \znk w_n \norm{x_n - \mu(C_k)}^2\ .
\end{align*}
Note that these quantities are random variables defined by the random process.
All of them depend on the binary random variables $\znk$.

\subsection{Proximity}

By definition, the binary random variables $z_{nk}$ have the property
\[ \E\left[ z_{nk}\right] = \Pr\left( z_{nk} = 1\right) = r_{nk}^m \enspace . \]

In the following, we use Chebyshev's and Markov's inequality to give concentration bounds on the difference between weights, means, and costs of the fuzzy clusters and the hard clusters constructed by the random process, respectively.
One might suspect that Chernoff bounds yield better results. 
Unfortunately, these bounds do not directly measure the differences between the means and costs in terms of the fuzzy $K$-means costs, respectively. 
Hence, it is not clear how Chernoff bounds can be applied here.

Let $\lambda, \nu \in \R_{>1}$ be constants.

\begin{lemma}[Weights]\label{lem:chebychev:weights}
	For all $k \OneTo{K}$ we have
	\[ \Pr\left( \abs{w(C_k) - R_k} \geq \lambda \eta_k \right) \leq \frac{1}{\lambda^2} \enspace , \]
	where
	\begin{align} \label{eq:eta_k}
	    \eta_k = \sqrt{\sum_{n=1}^N r_{nk}^m (1-r_{nk}^m) w_n^2} \enspace . 
	\end{align}
\end{lemma}
\begin{proof}
	Since $w(C_k) = \sum_{n=1}^N z_{nk} w_n$, we have
	$\E\left[ w(C_k) \right] = \sum_{n=1}^N \E \left[ z_{nk} \right] w_n = \sum_{n=1}^N r_{nk}^m w_n = R_k$.
	Furthermore, since $\{z_{nk}\}_n$ is a set of independent random variables, we have 
	\[ \Var\left( w(C_k) \right) = \sum_{n=1}^N \Var \left( z_{nk} \right) w_n^2=\sum_{n=1}^N  r_{nk}^m \left( 1-r_{nk}^m \right) w_n^2=\eta_k^2\enspace.\]
	The claim is a direct consequence of Chebyshev's inequality.
	
\end{proof}
Note that numerator and denominator of
\begin{align}
	\norm{\mu(C_k) - \mu_k}^2 
	= \frac{\norm{\sum_{n=1}^N z_{nk} w_n (x_n - \mu_k)}^2}{w(C_k)^2} \label{eq:diff-means}
\end{align}
are both random variables depending on the same random variables $z_{nk}$.
Lemma~\ref{lem:chebychev:weights} already gives a bound on the denominator $w(C_k)$. 
Next, we bound the numerator.
\begin{lemma}[Means]\label{lem:markov:means}
	\[ \Pr \left(\norm{\textstyle{\sum_{n=1}^N} z_{nk} w_n (x_n - \mu_k)}^2  \geq \nu \tau_k \right) \leq \frac{1}{\nu} \enspace , \]
	where
	\begin{align} \label{eq:tau_k}
	  \tau_k = \sum_{n=1}^N r_{nk}^m \left( 1-r_{nk}^m\right) w_n^2 \norm{x_n - \mu_k}^2\enspace . 
	\end{align}
\end{lemma}
\begin{proof} 
	Let $M_k = \norm{\sum_{n=1}^N z_{nk} w_n (x_n - \mu_k)}^2$.
	Observe that
	\begin{align*}
 		M_k = &\sca{\sum_{n=1}^N z_{nk} w_n (x_n - \mu_k)}{\sum_{n=1}^N z_{nk} w_n (x_n - \mu_k)} \\
 		= &\sum_{n=1}^N \sum_{o=1}^N z_{nk} z_{ok} w_n w_o \sca{x_n - \mu_k}{x_o - \mu_k} \enspace .
	\end{align*}
	Because the expectation is linear, we obtain
	\begin{align*}
		E[M_k] = &\sum_{n=1}^N \sum_{o=1}^N E\left[z_{nk} z_{ok}\right]  w_n w_o\sca{(x_n - \mu_k)}{(x_o - \mu_k)} \\
		= &\sum_{n=1}^N E\left[z_{nk}^2\right]  w_n^2 \norm{x_n - \mu_k}^2 + 
		\sum_{o\neq n} E\left[z_{nk} z_{ok}\right]  w_n w_o\sca{(x_n - \mu_k)}{(x_o - \mu_k)} \enspace .
	\end{align*}
	Recall that the $z_{nk}$ are independent binary random variables, hence for all $n,o\OneTo{N}$, $n\neq o$,
	\begin{align*}
	E\left[z_{nk}^2\right] = &\Pr\left(z_{nk} = 1 \right) = r_{nk}^m \\
	E\left[z_{nk} z_{ok} \right] = &\Pr\left(z_{nk}z_{ok} = 1\right) = r_{nk}^m r_{ok}^m \enspace . 
	\end{align*}
	Thus,
	\begin{align*}
		E[M_k] = &\sum_{n=1}^N r_{nk}^m  w_n^2 \norm{x_n - \mu_k}^2 + 
		\sum_{o\neq n} r_{nk}^m r_{ok}^m  w_n  w_o\sca{(x_n - \mu_k)}{(x_o - \mu_k)} \\
		= &\sum_{n=1}^N r_{nk}^m w_n^2 \norm{x_n - \mu_k}^2 - r_{nk}^{2m} \norm{x_n - \mu_k}^2 
		+\sum_{o=1}^N r_{nk}^m r_{ok}^m  w_n w_o\sca{(x_n - \mu_k)}{(x_o - \mu_k)} \\
		= &\sum_{n=1}^N r_{nk}^m (1-r_{nk}^m)  w_n^2 \norm{x_n - \mu_k}^2 + 
		r_{nk}^m w_n\sca{(x_n - \mu_k)}{\underbrace{\sum_{o=1}^N r_{ok}^m  w_o (x_o - \mu_k)}_{=0}} \\
		= &\sum_{n=1}^N r_{nk}^m (1-r_{nk}^m)  w_n^2 \norm{x_n - \mu_k}^2 = \tau_k \enspace.
	\end{align*}
	Applying Markov's inequality yields the claim.
%
\end{proof}
Finally, we can bound the cluster-wise cost as follows.
\begin{lemma}[Cost]\label{lem:markov:covar}
	For all $k \OneTo{K}$ we have
	\[ 
		\Pr \left( \km(C_k)   \geq  \nu \cdot \phiXmk(\{\rnk\}_n)\right) \leq \frac{1}{\nu}
	\]
\end{lemma}
\begin{proof}

	Observe that, by definition of $\mu(C_k)$ and Lemma~\ref{lem:magic-formula},
	\begin{align*}
		\km(C_k)  
		= \sum_{n=1}^N z_{nk} w_n\norm{x_n-\mu(C_k)}^2
		&\leq \sum_{n=1}^N z_{nk} w_n\norm{x_n-\mu_k}^2
	\end{align*}
	The expectation of the upper bound evaluates to
	\[ \E\left[  \sum_{n=1}^N z_{nk}  w_n \norm{x_n-\mu_k}^2  \right]  
	=  \sum_{n=1}^N r_{nk}^m  w_n\norm{x_n-\mu_k}^2 = \phiXmk(\{\rnk\}_n) \enspace .\]
	Applying Markovs's inequality yields the claim. 
	
\end{proof}

Now, we can formally prove the existence of hard clusters imitating given fuzzy clusters.

\begin{corollary}\label{cor:hard-clustering:without-assumption}
	Let $X=\{(x_n,w_n)\}_{n\OneTo{N}}\subset \R^D\times \R_{\geq 0}$ be a weighted point set. 
	Let $\{r_{nk}\}_{n,k}$ be the memberships of a fuzzy $K$-means solution for $X$ with corresponding optimal means $\{\mu_k\}_k$.
	Then, there exist pairwise disjoint subsets $\{C_k\}_{k\OneTo{K}}$ of $X$ such that for all $k\OneTo{K}$
	\begin{align}
		\abs{w(C_k) - R_k} &\leq \sqrt{4K} \cdot \eta_k  \label{eq:hard-cluster:wa:weights}\\
		\norm{\mu(C_k) - \mu_k}^2 &\leq \frac{4K }{(R_k - \sqrt{4K} \eta_k)^2} \cdot \tau_k \label{eq:hard-cluster:wa:means}\\
		\km(C_k)
		&\leq 4K \cdot \phiXmk(\{\rnk\}_n) \enspace ,\label{eq:hard-cluster:wa:costs}
	\end{align}
	where $\eta_k = \sqrt{\sum_{n=1}^N r_{nk}^m (1-r_{nk}^m) w_n^2}$ 
	and 
	$\tau_k = \sum_{n=1}^N r_{nk}^m \left( 1-r_{nk}^m\right) w_n^2 \norm{x_n - \mu_k}^2 $.

\end{corollary}
\begin{proof}
	Recall that the binary random variables $\znk$ indicate hard clusters $\{C_k\}_{k\OneTo{K}}$ by means of $C_k=\{(x_n,w_n)\in X|\znk=1\}$. 
	If we apply Lemma~\ref{lem:chebychev:weights} through \ref{lem:markov:covar} with $\lambda = \sqrt{4K}$ and $\nu=\lambda^2$, we can take the union bound and obtain that the inequalities stated in the lemmata hold simultaneously with probability strictly larger than $0$. 
	Finally, using Equation~\eqref{eq:diff-means} yields the claim.
\end{proof}

\subsection{Proof of Theorem~\ref{thm:hard-clustering}}

	We apply Corollary~\ref{cor:hard-clustering:without-assumption} to the given membership values $\rnkSet$ and the point set 
	\[ \hat X=\left\{\left(x_n,\frac{w_n}{\wmax}\right)\;\middle\vert \; (x_n,w_n)\in X\right\}\ .\]
	Let $k\OneTo{K}$. 
	Note that the cluster weight $\hat R_k$ with respect to $\hat X$ is given by 
	\begin{align*} 
	 \hat R_k = {\sum_{n=1}^N  r_{nk}^m \left(\frac{w_n}{\wmax}\right)} = \frac{1}{\wmax} R_k \ . 
	\end{align*}
	Using $\min_{k\OneTo{K}} R_k\geq 16K\wmax/\epsilon$, we can conclude
	\begin{align} 
	  \epsilon \cdot \hat R_k \geq 16K \ . \label{eq:R_k-geq-24Kwmax/eps}
	\end{align}

	Observe that
	\begin{align} \eta_k = \sqrt{\sum_{n=1}^N  r_{nk}^m (1-r_{nk}^m) \left(\frac{w_n}{\wmax}\right)^2 } 
	\leq \sqrt{\sum_{n=1}^N  r_{nk}^m \frac{w_n}{\wmax} }
	= \sqrt{\hat R_k} 
	\enspace . \label{eq:eta_k-geq-sqrt-R_k}
	\end{align}
	Due to Inequality~\eqref{eq:R_k-geq-24Kwmax/eps}, we have
	\begin{align}
	 \sqrt{4K} 
	 \leq \frac{\sqrt{16K}}{2} 
	 < \frac{\sqrt{\epsilon\cdot \hat R_k}}{2} 
	 < \frac{\sqrt{\hat R_k}}{2} \label{eq:Rk-geq-6clambda2}\ .
	\end{align}
	Using Inequality~\eqref{eq:Rk-geq-6clambda2} and \eqref{eq:eta_k-geq-sqrt-R_k}, we can conclude
	\begin{align} 
	 \sqrt{4K}\cdot \eta_k < \frac{\hat R_k}{2} \enspace . \label{eq:lambda-etak}
	\end{align}
	Hence, Inequality~\eqref{eq:hard-cluster:wa:weights} of Corollary~\ref{cor:hard-clustering:without-assumption} yields Inequality~\eqref{eq:proof:r} of Theorem~\ref{thm:hard-clustering}. 
			
	Note that the optimal mean vectors $\mu_k$ with respect to $X$ and $\rnkSet$ coincide with the corresponding optimal mean vectors with respect to $\hat X$ and $\rnkSet$ (cf. Equation~\eqref{eq:opt-means}).
	 
	Next, we have 
	\begin{align*} 
	\tau_k 
	= \sum_{n=1}^N  r_{nk}^m \left( 1-r_{nk}^m\right) \left(\frac{w_n}{\wmax}\right)^2 \norm{x_n - \mu_k}^2 
	\leq \frac{1}{\wmax}\phiXmk(\{\rnk\}_n)
	\enspace .  
	\end{align*}
	Using Inequality~\eqref{eq:lambda-etak}, we obtain 
	\[ \hat R_k - \sqrt{4K}\eta_k \geq  \frac{\hat R_k}{2} > 0 \ .\] 
	Due to Inequality~\eqref{eq:Rk-geq-6clambda2}, we have 
	\[ 4K\leq  \frac{\epsilon \hat R_k}{4}\ .\] 
	Hence, 
	\[ \frac{4K}{(\hat R_k - \sqrt{4K}\eta_k)^2} 
	\leq \frac{\epsilon}{\hat R_k } 
	= \frac{\epsilon\cdot \wmax}{R_k} \enspace .\]
	Therefore, Inequality~\eqref{eq:hard-cluster:wa:means} of Corollary~\ref{cor:hard-clustering:without-assumption} yields Inequality~\eqref{eq:proof:mu} of Theorem~\ref{thm:hard-clustering}. 
	
	Finally, recall that the optimal mean vectors for $X$ and $\rnkSet$ coincide with those for $\hat X$ and $\rnkSet$. 
	Thus, $\phim_{\hat X,k}(\{\rnk\}_n) = \frac{1}{\wmax}\phiXmk(\{\rnk\}_n)$. 
	Analogously, for all $C\subset X$ and $\hat C\subset \hat X$ with $\left\{\left(x,\frac{w}{\wmax}\right)\middle|(x,w)\in C\right\}=\hat C$ we have $\km(\hat C)=\frac{1}{\wmax}\km(C)$. 
	Hence, Inequality~\eqref{eq:hard-cluster:wa:costs} of Corollary~\ref{cor:hard-clustering:without-assumption} yields Inequality~\eqref{eq:proof:sigma} of Theorem~\ref{thm:hard-clustering}.
		
\qed

%% file: app_superset-sampling.tex
\subsection{Superset Sampling (Proof of Theorem~\ref{thm:superset-sampling})}\label{sec:superset-sample}

Recall that, in Section~\ref{sec:stochastic-fuzzy-k-means}, we argued on the existence of hard clusters which approximate optimal fuzzy clusters well.
In this section we consider the problem of finding a good approximation to the means of such unknown hard clusters.

First, consider the problem of finding the mean of a single unknown cluster $C$.
That is, given a set $X\subset\R^D$ and some unknown subset $C\subset X$, we want to find a good approximation to $\mu(C)$.
If we assume that $C$ contains at least a constant fraction of the points of $X$, then this problem can be solved via the \emph{superset sampling} technique \cite{inaba94}.
The main idea behind this technique is that the mean of a small uniform sample of a set is, with high probability, already a good approximation to the mean of the whole set.
Knowing that $C$ contains a constant fraction of points from $X$, we can obtain a uniform sample of $C$ by sampling a uniform multiset $S$ from $X$ and inspecting all subsets of $S$.
Thereby, we obtain a set of candidate means, i.e. the means of all subsets of $S$, including (with certain probability) one candidate that is a good approximation to the mean of $C$. 
Formally, using \cite{ackermann10}, we directly obtain the following lemma.

\begin{lemma}[Superset Sampling \cite{ackermann10}]
 Let $X\subset\R^D$, $\alpha\in(0,1]$, and $\epsilon\in(0,1]$.
 Let $S\subset X$ be a uniform sample multiset of size at least $4/(\alpha\epsilon)$.
 
 Consider an arbitrary but fixed (unknown) subset $C\subset X$ with $\abs{C}\geq\alpha \abs{X}$.
 With probability at least $1/10$ there exists a subset $C'\subset S$ with $\abs{C'}= \lceil 2/\epsilon \rceil$ satisfying
 \begin{align*} 
 \norm{\mu(C) - \mu(C')}^2 \leq \frac{\epsilon}{\abs{C}}\sum_{x\in C} \norm{x-\mu(C)}^2\enspace.
 \end{align*}
 where for any finite set $S\subset \R^D$ we set $ \mu(S)\coloneqq \frac{\sum_{x\in S} x}{\abs{S}}$.
\end{lemma}

As a consequence, for weighted point sets $X$ with weights in $\Q$, we obtain a good approximation of the mean of $X$ by sampling each (unweighted) point with probability proportional to its weight.

\begin{corollary}[Weighted Superset Sampling]\label{lem:superset-sampling}
 Let $X=\{(x_n,w_n)\}_{n\OneTo{N}}\subset\R^D\times \Q_{\geq 0}$,  $\alpha\in(0,1]$, and $\epsilon\in(0,1]$.
 Let $W=\sum_{n=1}^N w_n$. 
 Let $S\subset \{(x_n,1)\}_{n\OneTo{N}}$ be a sample multiset of size at least $4/(\alpha\epsilon)$,
 where each point $x_n\in X$ is sampled with probability $w_n/W$.
 
 Consider an arbitrary but fixed (unknown) subset $C\subset X$ with $w(C)\geq\alpha W$.
 With probability at least $1/10$ there exists a subset $C'\subset S$ with $\abs{C'}=\lceil 2/\epsilon \rceil$ satisfying
 \begin{align*} 
 \norm{\mu(C) - \mu(C')}^2 \leq \frac{\epsilon}{w(C)}\km(C)\enspace.
 \end{align*}
\end{corollary}
\begin{proof}
 Let $\omega$ be a common denominator of all $\{w_n\}_{n\OneTo{N}}$. 
 Let $\hat X\subset \R^D$ be the multiset containing $w_n\cdot \omega$ copies of each $x_n$ with $(x_n,w_n)\in X$, and 
 let $\hat C\subset \R^D$ be the multiset containing $w_n\cdot \omega$ copies of each $x_n$ with $(x_n,w_n)\in C$. 
 Note that sampling a point uniformly at random from $\hat X$ yields the same distribution on the points as sampling a point from $X$ with probability $w_n/W$.
 Furthermore, if $C\subset X$ with $w(C)\geq\alpha W$, then $\hat{\abs{C}} = w(C) \geq \alpha W = \alpha \hat{\abs{X}}$. 
 Hence, applying the previous lemma directly yields the claim. 
\end{proof}

Given Corollary~\ref{lem:superset-sampling}, we can finally prove Theorem~\ref{thm:superset-sampling}. 

\begin{proof}[Proof of Theorem~\ref{thm:superset-sampling}]
  
	Let $R= \lceil 10\log(2K)\rceil$ and let $W=\sum_{n=1}^N w_n$. 
	For each $r\OneTo{R}$, sample an unweighted multiset $S_{r}\subset X$ of size at least $4/(\alpha\epsilon)$ by choosing a point $x_n\in X$ with probability $w_n/W$. 
	Define the candidate set $T$ as
	\[  T \coloneqq \left\{ \mu(S')\ \bigg|\ S'\subset S_{r},\ \abs{S'}=\lceil 2/\epsilon \rceil,\ r\OneTo{R}\right\}^K \enspace . \]
	
	Next, we prove that $T$ and its construction have the desired properties. 
	
	Let $M = \left\{ \mu(S')\ |\ S'\subset S_{r},\ \abs{S'}=\lceil 2/\epsilon\rceil,\ r\OneTo{R}\right\}$ and fix an arbitrary $k\OneTo{K}$ with $w(C_k)\geq \alpha W$. 
	By Lemma~\ref{lem:superset-sampling}, there is, with probability $p\coloneqq 1-(9/10)^R$, a candidate $\tmu_k\in M$ satisfying 
	\[ \norm{\mu(C_k) - \tmu_k}^2 \leq \frac{\epsilon}{w(C_k)}\km(C_k)\enspace  .\]
	Since $ R \geq 10\ln(2K)$, we have that $(9/10)^R \leq 1/(2K)$.
	Hence, $p\geq 1-1/(2K)$.
	
	By taking the union bound, we obtain that,  with probability at least $1/2$,  $T=M^K$ contains a tuple $(\tmu_k)_{k\OneTo{K}}$ where for each $k\OneTo{K}$ the vector $\tmu_k$ is close to $\mu(C_k)$ if $w(C_k)\geq \alpha W$.
	
	Finally, we analyze the time needed to construct $T$.
	We have to sample $R$ multisets of size $\lceil 4/(\alpha\epsilon)\rceil$ from $X$, which needs running time $\mathcal{O}(R\cdot (1/(\alpha\epsilon))\cdot N)$.
	Each of the multisets contains at most $\left(4/(\alpha\epsilon)\right)^{2/\epsilon} $ subsets of size $\lceil 2/\epsilon\rceil$.
	We compute the the means of all these subsets, which needs time $\mathcal{O}(1/\epsilon)$ per subset.
	Then, by combining all these means we obtain the candidate set of tuples $T$.
	Consequently, this set has size
	\[\abs{T} \leq \left( R \cdot \left(\frac{4}{\alpha\epsilon}\right)^{1/\epsilon} \right)^K
	          = 2^{\frac{1}{\epsilon}\cdot\log \left(\frac{4}{\alpha\epsilon}\right)\cdot \log(R)}\enspace .\]  
	Hence, we can bound the running time of our construction by
	\[ \mathcal{O} \left( R\cdot \frac{1}{\epsilon} \left(N 
	+ 2^{\frac{K}{\epsilon}\cdot\log \left(\frac{1}{\alpha\epsilon}\right)\cdot \log(R)} \right)\right) \enspace .\]

\end{proof}

%% file: app_candidate-means.tex
\subsection{Candidate Set of Means (Proof of Theorem~\ref{thm:candidate-means})}\label{sec:proof:candidates}

The following we use the coreset construction used by \cite{Chen09}, \cite{HarPeled03}, and \cite{Matousek00} to obtain a candidate set that contains a set of means which induce an  $(1+\epsilon)$-approximation to the fuzzy $K$-means problem. 

\subsubsection{Construction}\label{sec:proof:candidates:construction}

We are given $X=\{x_n\}_{n\OneTo{N}}\subset\R^D$ and $K\in\N$.
Let $A=\{a_k\}_{k\OneTo{K}}\subset\R^D$ be an $\alpha$-approximation of the $K$-means problem with respect to $X$, i.e. 
\begin{align}\label{eq:candidate:A-approx}
 \km_X(A) \leq \alpha\cdot\km_{X,K}\ .
\end{align}
Furthermore, let
\begin{align} 
 R 
 &\coloneqq \sqrt{\frac{\km_X(A)}{\alpha N}}\ \label{eq:candidate:R}\ ,\\
 \ball(x,r)
 &\coloneqq \left\{ y\in\R^D \;|\; \norm{x-y} \leq r \right\} \ ,\\
 \Phi 
 &\coloneqq \left\lceil \frac{1}{2}\left( \log(\alpha N) + m\cdot\log\left( \frac{64\alpha m K^2}{\epsilon}\right) \right) \right\rceil\ , \label{eq:candidate:Phi}\\
 \cU
 &\coloneqq\bigcup_{k=1}^K \ball(a_k,2^\Phi R) \ ,\\
 \Lkj
 &\coloneqq \begin{cases}
     \ball(a_k,R), & \mbox{ if }j=0\\
     \ball(a_k,2^jR)\setminus \ball(a_k,2^{j-1}R) &\mbox{ if }j\geq 1
    \end{cases}\ , \\
 \kappa
 &\coloneqq \alpha K^{m-1} \ , \mbox{ and}\\
 b
 &\coloneqq 1208
\end{align}
for all $k\OneTo{K}$, $j\in\{0,\ldots,\Phi\}$, $x\in\R^D$, and $r\in\R$. 

Construct an axis-parallel grid with side length 
\[ \rho_j= \frac{2^j\epsilon R }{b\kappa \sqrt{D}} \] 
to partition $\Lkj$ into cells. 
Inside each $\Lkj$ pick the center of the cell as its representative point. 
Denote by $\cG$ be the set of all representative points. 

\subsubsection{Proof of Theorem~\ref{thm:candidate-means}}

In Section~\ref{sec:proof:candidates:existence}, we show that there exists $\tC=\{\tmu_l\}_{l\OneTo{L}}\subseteq\cU$ with $L\OneTo{K}$ and 
\[ \phiXm(\tC)\leq \left(1+\frac{\epsilon}{2}\right) \phiopt\ .\]
Since $\tC=\{\tmu_l\}_{l\OneTo{L}}\subseteq\cU$, there our construction from Section~\ref{sec:proof:candidates:construction} defines a representative $\tmu_l'\in\cG$ for each $\tmu_l\in \tC$. 

In Section~\ref{sec:proof:candidates:kmeans-approx}, we consider arbitrary sets $C=\cU$ with representatives $C'\subseteq\cG$. 
We show that $K$-means costs of $C$ and $C'$ are similar. 
Given this result, in Section~\ref{sec:proof:candidates:fuzzy-approx}, we show that the fuzzy $K$-means costs of $C$ and $C'$ are similar.
More precisely, we show that if $\alpha \leq 2$, then
\[ \abs{\phiXm(C)-\phiXm(C')}\leq \frac{\epsilon}{4}\phiXm(C)\ . \]
In particular, this holds for $\tC$ and its representatives $\tC'\subset \cG$. 

From these results we can conclude that
\[ 
\phiXm(\tC')
\leq \left(1+\frac{\epsilon}{4}\right)\phiXm(\tC) 
\leq \left(1+\frac{\epsilon}{2}\right)\left(1+\frac{\epsilon}{4}\right) \phiopt 
\leq (1+\epsilon)\phiopt\ .\]
Finally, observe that for all $M\subset\R^D$ it holds $\phiXm(\tC'\cup M)\leq \phiXm(\tC')$. 
Thus, by testing all subsets $\{\mu_k'\}_{k\OneTo{K}}\subset\cG$, we find a solution that is at least as good as $\tC'$ and hence a $(1+\epsilon)$-approximation.

\paragraph{Preliminaries}

In the following proof, we make extensive use of the following lemmata, as well as the lemmata given in Section~\ref{sec:stuff}. 

\begin{corollary}\label{cor:kmeans}
   Let $X\subset \R^D\times \R_{\geq 0}$ and $C\subset\R^D$ with $\abs{C}=K$. 
   If $\km_X(C)\leq \gamma \km_{X,K}$, then $\phiXm(C)\leq (\gamma\cdot K^{m-1})\phiopt$. 
   If $\phiXm(C)\leq \gamma \phiopt$, then $\km_X(C)\leq (\gamma\cdot K^{m-1}) \km_{X,K}$	
\end{corollary}
\begin{proof}
 Use Lemma~\ref{lem:kmeans}.
\end{proof}

\begin{lemma}\label{lem:candidate:cK}
For all $C\subset\R^D$, $\abs{C}=K$, we have
\begin{align*}
 N\cdot R^2 
 = \sum_{n=1}^N  d(x_n,A)^2
 =\km_X(A)
 \leq \alpha \cdot \km_{X,K}
 \leq \kappa \cdot \phiopt
 \leq \kappa \cdot \phiXm(C) \ .
\end{align*}
\end{lemma}
\begin{proof}
 Use Definition~\ref{def:kmeans}, Equation~\eqref{eq:candidate:R}, Equation~\eqref{eq:candidate:A-approx}, and Corollary~\ref{cor:kmeans}.
\end{proof}

\begin{lemma}\label{lem:candidates:mu-mu'}
 For each $\mu\in\cU$ with representative $\mu'\in \cG$ it holds
 \[ \norm{\mu-\mu'} \leq  
\frac{2\epsilon}{b\kappa} \left( d\left(\tmu,A\right) + R \right) 
\leq \frac{2\epsilon}{b\kappa} \left( \norm{x-\tmu}+ d(x,A) + R \right) \ . \] 
and
\begin{align*}\norm{\mu-\mu'}^2 
 \leq \frac{12\epsilon^2}{b^2\kappa^2} \left( \norm{x-\tmu}^2 + d(x,A)^2 + R^2 \right) 
\end{align*}
for all $x\in \R^D$ and $\tmu\in\{\mu,\mu'\}$.  
\end{lemma}
\begin{proof}
 By construction, some $\Lkj$ contains $\mu$ and its representative $\mu'$. 
 Moreover, $\mu$ and $\mu'$ are contained in the same grid cell with side length $\rho_j$.
 Hence, $\norm{\mu-\mu'}\leq \frac{2^j\epsilon R }{b\kappa}$. 
 By construction of $\Lkj$, we have $\min\{d(\mu,A),d(\mu',A)\}\geq 2^{j-1}R $ for all $x_n\in\Lkj$ with  $j\geq 1$.
 For $j=0$ we know that $\norm{\mu-\mu'}\leq \frac{2^0\epsilon R }{b\alpha}= \frac{\epsilon R }{b\kappa}$.
 Applying Lemma~\ref{lem:squares} and observing that for all $x,y \in\R^D$ and $C\subset\R^D$ we have $d(x,C)\leq \norm{x-y}+d(y,C)$, yields the claim. 
\end{proof}

\paragraph{Existence of an  \texorpdfstring{$(1+\frac{\epsilon}{2})$}{(1+eps/2)}-Approximation in \texorpdfstring{$\cU$}{U}}\label{sec:proof:candidates:existence}

\begin{claim}\label{claim:candidate:exist-approx-in-cU}
There exists  $\tC=\tmuSet\subset \cU$ with 
 \[ \phiXm(\tC)\leq \left(1+\frac{\epsilon}{2}\right)\phiopt\ .\] 
\end{claim}

In the following, we prove Claim~\ref{claim:candidate:exist-approx-in-cU}.

\begin{claim}\label{claim:candidate:subseteqU}
  \[ \bigcup_{x\in X} \ball(x,r) \subseteq  \bigcup_{k=1}^K \ball(a_k,2^\Phi R) = \cU  \]
  where
  \[ r = \sqrt{\left(1+\frac{\epsilon}{2}\right) \left(\frac{8mK^2}{\epsilon} \right)^m \km_X(A)}\ .\]
\end{claim}
\begin{proof}
 Towards a contradiction, assume that there exists an $x\in X$ with $\ball(x,r) \nsubseteq \cU$. 
 By definition of $\cU$, this implies that for all $k\OneTo{K}$ we have $\ball(x,r)\nsubseteq \ball(a_k,2^\Phi R)$. 
 Hence, 
 \begin{align*} 
 &d(x,A)\\
 &\geq  2^\Phi R-r \\
 &=  2^{\Phi} \sqrt{\frac{\km_X(A)}{\alpha N}} -  \sqrt{\left(1+\frac{\epsilon}{2}\right) \left(\frac{8mK^2}{\epsilon} \right)^m \km_X(A)} \tag{by Equation~\eqref{eq:candidate:R}}\\
 &=  \sqrt{\km_X(A)}
 \left( 
 2^{\Phi} \sqrt{\frac{1}{\alpha N}} 
 -  \sqrt{\left(1+\frac{\epsilon}{2}\right) \left(\frac{8mK^2}{\epsilon} \right)^m }
 \right)\ .
 \end{align*}
 Observe that by Equation~\eqref{eq:candidate:Phi}, we have
 \begin{align*}
 \Phi
 &= \frac{1}{2}\left( \log(\alpha N) + m\cdot\log\left( \frac{64\alpha m K^2}{\epsilon}\right) \right)\\
 &=  
 \log\left(\sqrt{\alpha N}\right) 
 + \log\left( \sqrt{ \left(\frac{64\alpha m K^2}{\epsilon} \right)^m}\right) \\
 &\geq 
 \log\left(\sqrt{\alpha N}\right) 
 + \log\left( \sqrt{ 2\cdot 2\alpha \cdot 2\left(\frac{8 m K^2}{\epsilon} \right)^m}\right) \\
 &= 
 \log\left(\sqrt{\alpha N} \cdot \left( \sqrt{ 2\cdot 2\alpha \cdot 2\left(\frac{8 m K^2}{\epsilon} \right)^m}\right)\right) \\
 &\geq \log\left(\sqrt{\alpha N}\cdot\left( \sqrt{2\alpha} + \sqrt{2 \left(\frac{8mK^2}{\epsilon} \right)^m } \right)\right)\tag{since $\sqrt{a+b}\leq 2\sqrt{ab}$ for all $a,b\geq 1$}\\
 &\geq \log\left(\sqrt{\alpha N}\cdot\left( \sqrt{2\alpha} + \sqrt{\left(1+\frac{\epsilon}{2}\right) \left(\frac{8mK^2}{\epsilon} \right)^m } \right)\right) \ . \tag{since $\epsilon<1$}
 \end{align*}
 Hence, 
 \[ d(x,A) \leq  \sqrt{\km_X(A)}\sqrt{2\alpha}\ .\]
 
 Using Definition~\ref{def:kmeans}, we can conclude
 \begin{align*}
  \km_X(A) 
  \geq d(x,A)^2 
  \geq 2\alpha \km_X(A)
  \geq 2\alpha \km_{X,K}
  > \alpha \km_{X,K}\ ,
 \end{align*}
 which contradicts Equation~\eqref{eq:candidate:A-approx}. 
\end{proof}

\begin{claim}\label{claim:candidate:muk-notin-cU-and-Ck-empty}
 Let $\tC=\{\tmu_k\}_{k\OneTo{K}}\subset \R^D$ such that
 \[ \phiXm(\tC) \leq \left(1+\frac{\epsilon}{4K}\right)^{\ell} \phiopt\ .\]
 Let $\{\tC_k\}_{k\OneTo{K}}$ be a partition of $X$ induced by $\tC$. 
 For all $k\OneTo{K}$ we have 
 \[ \left( \mu_k\notin \cU\ \wedge\ C_k=\emptyset \right)\ \Rightarrow\  \phiXm(\tC\setminus\{\mu_k\})\leq \left(1+\frac{\epsilon}{4K}\right)^{\ell+1}\phiXm(\tC)\ .\]
\end{claim}
\begin{proof}
 Let $\rnkSet$ be the optimal responsibilities induced by $\tC$. 
 
 Using Claim~\ref{claim:candidate:subseteqU} and $\mu_k\notin\cU$, we obtain
 \begin{align*} 
    \left(1+\frac{\epsilon}{4K}\right)^\ell \phiopt
    \geq \phiXm(\tC) 
    \geq \phiXmk(\tC)
    = \sum_{n=1}^N \rnk^m \norm{x_n-\mu_k}^2 
    \geq \left( \sum_{n=1}^N \rnk^m \right) r^2 \ ,
 \end{align*}
 where $r$ is defined as in Claim~\ref{claim:candidate:subseteqU}. 
 Hence, 
 \begin{align*} 
 \sum_{n=1}^N \rnk^m 
 &\leq \frac{\left(1+\frac{\epsilon}{4K}\right)^\ell  \phiopt}{r^2 } \\
 &\leq \frac{\left(1+\frac{\epsilon}{4K}\right)^\ell  \km_X(A)}{r^2 }\tag{Lemma~\ref{lem:kmeans}} \\
 &= \frac{\left(1+\frac{\epsilon}{4K}\right)^K \km_X(A)}{ \left(1+\frac{\epsilon}{2}\right) \left(\frac{8mK^2}{\epsilon} \right)^m \km_X(A) } \tag{Claim~\ref{claim:candidate:subseteqU}} \\
 &\leq \frac{\left(1+\frac{\epsilon}{2}\right) }{ \left(1+\frac{\epsilon}{2}\right) \left(\frac{8mK^2}{\epsilon} \right)^m  } \tag{Lemma~\ref{lem:stuff:1+eps}}\\
 &\leq \left(\frac{\epsilon}{8mK^2} \right)^m \ .
 \end{align*}
 Consequently, 
 \begin{align} 
 \rnk\leq \frac{\epsilon}{8mK^2} \ . \label{eq:candidates:rnk-geq-eps-div-2mK}
 \end{align}
  
 Since $C_k=\emptyset$, for all $n\OneTo{N}$ 
 there exists an $l(n)\OneTo{K}$ with $l(n)\neq k$ and 
 \[ r_{n,l(n)}\geq \frac{1}{K}\ .\]
 Using Equation~\eqref{eq:candidates:rnk-geq-eps-div-2mK}, we can conclude
 \[ \rnk \leq \frac{\epsilon}{8mK^2} \leq \frac{\epsilon}{8mK} r_{n,l(n)}\ .\]
 Using Lemma~\ref{lem:stuff:1+eps}, we obtain
 \begin{align*} 
 \left(\rnk+r_{n,l(n)}\right)^m 
 \leq \left(\left(1+\frac{\epsilon}{8Km}\right) r_{n,l(n)} \right)^m 
 \leq \left(1+\frac{\epsilon}{4K}\right) r_{n,l(n)}^m\ . 
 \end{align*}
 Hence, we have 
 \begin{align*} 
   \phiXm(\tC\setminus\{\tmu_k\})
   &\leq  \sum_{n=1}^N \sum_{l\neq l(n),k}\rnk^m \norm{x_n-\tmu_l}^2 +
    \left(r_{n,l(n)}+\rnk\right)^m \norm{x_n-\tmu_{l(n)}}^2 \\
   &\leq  \sum_{n=1}^N \sum_{l\neq l(n),k}\rnk^m \norm{x_n-\tmu_l}^2 +
   \left(1+\frac{\epsilon}{4K}\right) r_{n,l(n)}^m \norm{x_n-\tmu_{l(n)}}^2 \\
   &\leq \left(1+\frac{\epsilon}{4K}\right) \phiXm(\tC) \\
   &\leq \left(1+\frac{\epsilon}{4K}\right)^{\ell+1} \phiopt 
 \end{align*}
 which yields the claim.
\end{proof}

\begin{claim}\label{claim:candidate:no-muk-notin-cU-and-Ck-empty}
 There exists $L\OneTo{K}$ and $\tC=\{\tmu_l\}_{l\OneTo{L}}\subset \R^D$ satisfying the following properties:
 \begin{enumerate}
  \item $\phiXm(\tC)\leq \left(1+\frac{\epsilon}{2}\right)\phiopt$
  \item Let $\{\tC_l\}_{l\OneTo{L}}$ be a partition of $X$ induced by $C$.   
  For all $l\OneTo{L}$ we have
 \[  \mu_l\notin \cU \ \Rightarrow\  \tC_l\neq\emptyset\ .\]
 \end{enumerate}

\end{claim}
\begin{proof}
  Consider an arbitrary but fixed optimal fuzzy $K$-means solution $O=\{o_k\}_{k=1}^K$. 
  There are at most $K-1$ means in $O$ that are not in $\cU$ \emph{and} where the corresponding clusters $O_k$ are empty. 
  By repeatedly applying Claim~\ref{claim:candidate:muk-notin-cU-and-Ck-empty}, we obtain a solution $\tilde O$ with $\abs{\tilde O}\leq K$ satisfying the second property in the claim. 
  Furthermore, we have
  $\phiXm(\tilde O) \leq (1+\frac{\epsilon}{4K})^K \phiopt \leq \left(1+\frac{\epsilon}{2}\right)\phiopt$, where the last inequality is due to Lemma~\ref{lem:stuff:1+eps}. 
  This yields the claim. 
\end{proof}

\begin{proof}[Proof of Claim~\ref{claim:candidate:exist-approx-in-cU}]
  Consider the solution $\tC=\{\tmu_l\}_{l\OneTo{L}}$ from Claim~\ref{claim:candidate:no-muk-notin-cU-and-Ck-empty}. 
  Assume $\tmu_l\notin \cU$. 
  Then, by the second property of Claim~\ref{claim:candidate:no-muk-notin-cU-and-Ck-empty}, the corresponding cluster $\tC_l$ is not empty. 
  Hence,
  \begin{align*}
   \phiXm(\tC) 
   &\geq \frac{1}{K^{m-1}}\cdot \km(\tC)\tag{Lemma~\ref{lem:kmeans} and $L\geq K$} \\
   &\geq \frac{1}{K^{m-1}}\cdot \sum_{x\in \tC_k} \norm{x_n-\tmu_k}^2 \\
   &\geq \frac{1}{K^{m-1}}\cdot \sum_{x\in \tC_k} r^2 \tag{by Claim~\ref{claim:candidate:subseteqU}}\\
   &\geq \frac{1}{K^{m-1}}\cdot  r^2 \tag{$C_k\neq\emptyset$}\\
   &\geq \frac{1}{K^{m-1}}\cdot  \left(1+\frac{\epsilon}{2}\right) \left(\frac{2mK^2}{\epsilon} \right)^m \km_X(A) \\
   &> \left(1+\frac{\epsilon}{2}\right) \phiopt \ ,\tag{by Lemma~\ref{lem:kmeans}} 
  \end{align*}
  which is a contradiction to the first property of Claim~\ref{claim:candidate:no-muk-notin-cU-and-Ck-empty}. 
  Hence, from the properties of Claim~\ref{claim:candidate:no-muk-notin-cU-and-Ck-empty}, we can conclude that $\tC\subset\cU$. 
\end{proof}

\paragraph{Notation}

For the following proofs fix an arbitrary solution
\[ C=\{\mu_k\}_{k\OneTo{K}}\subset \cU\subset \R^D\ . \]
Let $\rnkSet$ be the optimal responsibilities induced by $C$. 
We denote the representative of $\mu_k\in C$ by $\mu_k'\in\cG$, and the set of all these representatives by 
\[ C'\coloneqq\{\mu_1',\ldots,\mu_K'\ |\ \mu_k' \mbox{ is representative of }\mu_k\in C\mbox{ for all }k\OneTo{K}\}\ .\]

\paragraph{Closeness with respect to the \texorpdfstring{$K$}{K}-Means Problem}\label{sec:proof:candidates:kmeans-approx}

\begin{claim}\label{claim:candidates:kmeans-approx}
\[ \km_X(C')\leq \gamma \km_X(C)\ , \]
 where 
 \[ \gamma = 1+\frac{72\epsilon}{bK^{m-1}}\ .\]
\end{claim}

In the following, we prove Claim~\ref{claim:candidates:kmeans-approx}
To this end, denote by $\dot\bigcup_{k\OneTo{K}} C_k=X$ and $\dot\bigcup_{k\OneTo{K}} C_k'=X$ the partitions (ties broken arbitrarily) induced by $C$ and $C'$, respectively. 

\begin{claim}\label{claim:candidates:kmeans-approx:1}
  If $\km_X(C)> \km_X(C')$, then
  \begin{align*}
    \abs{\km_X(C)-\km_X(C')} \leq 
    \sum_{k=1}^K  \abs{C_k'} \norm{\mu_k-\mu_k'}^2 + 2 \sum_{x_n\in C_k'} \norm{\mu_k-\mu_k'} \norm{x_n-\mu_k'}^2\ ,
  \end{align*}
  otherwise
  \begin{align*}
    \abs{\km_X(C)-\km_X(C')} \leq 
    \sum_{k=1}^K  \abs{C_k} \norm{\mu_k-\mu_k'}^2 + 2 \sum_{x_n\in C_k} \norm{\mu_k-\mu_k'} \norm{x_n-\mu_k}^2 \ .
  \end{align*}
\end{claim}
\begin{proof}
 If $\km_X(C)\geq\km_X(C')$, then
 \begin{align*}
 \abs{\km_X(C)-\km_X(C')}
 &= \sum_{k=1}^K \sum_{x_n\in C_k} \norm{x_n-\mu_k}^2 - \sum_{x_n\in C_k'} \norm{x_n-\mu_k'}^2\\
 &\leq \sum_{k=1}^K \sum_{x_n\in C_k'} \norm{x_n-\mu_k}^2 - \norm{x_n-\mu_k'}^2 \tag{$\{C_k\}_k$ induced by $C$}\\
 &\leq  \sum_{k=1}^K \sum_{x_n\in C_k'} \norm{\mu_k-\mu_k'}^2 + 2 \norm{\mu_k-\mu_k'} \norm{x_n-\mu_k'}^2 \tag{Lemma~\ref{lem:abc}}\\
 \end{align*}
 If $\km_X(C)<\km_X(C')$, then the term can be bounded analogously. 
 This yields the claim. 
\end{proof}

\begin{claim}\label{claim:candidates:kmeans-approx:2}
\[ \sum_{k=1}^K \abs{C_k} \norm{\mu_k-\mu_k'}^2 
\leq \frac{36\epsilon^2}{b^2\kappa}  \km_X(C)\]
\end{claim}
\begin{proof}
  \begin{align*}
  \sum_{k=1}^K \abs{C_k} \norm{\mu_k-\mu_k'}^2 
  &\leq  \frac{12\epsilon^2}{b^2\kappa}  \sum_{k=1}^K \sum_{x_n\in C_k}  \left( \norm{x_n-\mu_k}^2 + d(x_n,A)^2 + R^2 \right) \tag{Lemma~\ref{lem:candidates:mu-mu'}}\\
  &=  \frac{12\epsilon^2}{b^2\kappa^2}  \left( \sum_{k=1}^K \sum_{x_n\in C_k}\norm{x_n-\mu_k}^2 +   \sum_{n=1}^N  d(x_n,A)^2 + N R^2 \right) \\
  &\leq  \frac{12\epsilon^2}{b^2\kappa^2}  \left( \km_X(C) +   \km_X(A) + N R^2 \right) \\
  &\leq  \frac{12\epsilon^2}{b^2\kappa^2}  \left( \km_X(C) +   2\alpha\km_{X,K}  \right) \tag{Lemma~\ref{lem:candidate:cK}}\\
  &\leq  \frac{36\epsilon^2}{b^2\kappa}  \km_X(C) \tag{since $\alpha \geq 1$}
\end{align*}
\end{proof}

\begin{claim}\label{claim:candidates:kmeans-approx:3}
\[ \sum_{k=1}^K \abs{C_k'} \norm{\mu_k-\mu_k'}^2 
\leq \frac{12\epsilon^2}{b^2\kappa^2}  \left(\km_X\left(C'\right) +   2\alpha \cdot\km_X(C) \right)\]
\end{claim}
\begin{proof}
  \begin{align*}
  \sum_{k=1}^K \abs{C_k'} \norm{\mu_k-\mu_k'}^2 
  &\leq  \frac{12\epsilon^2}{b^2\kappa^2}  \sum_{k=1}^K \sum_{x_n\in C_k'}  \left( \norm{x_n-\mu_k'}^2 + d(x_n,A)^2 + R^2 \right) \tag{Lemma~\ref{lem:candidates:mu-mu'}}\\
  &=  \frac{12\epsilon^2}{b^2\kappa^2}  \left( \sum_{k=1}^K \sum_{x_n\in C_k'}\norm{x_n-\mu_k'}^2 +   \sum_{n=1}^N  d(x_n,A)^2 + N R^2 \right) \\
  &\leq  \frac{12\epsilon^2}{b^2\kappa^2}  \left( \km_X(C') +   \km_X(A) + N R^2 \right) \\
  &\leq  \frac{12\epsilon^2}{b^2\kappa^2}  \left(\km_X(C') +   2\alpha \cdot\km_{X,K} \right) \tag{Lemma~\ref{lem:candidate:cK}}\\
  &\leq  \frac{12\epsilon^2}{b^2\kappa^2}  \left(\km_X(C') +   2\alpha \cdot\km_X(C) \right)
\end{align*}
\end{proof}

\begin{claim}\label{claim:candidates:kmeans-approx:4}
 \begin{align*}
 2\sum_{k=1}^K \sum_{x_n\in C_k}    \norm{\mu_k-\mu_k'} \norm{x_n-\mu_k}
 \leq  \frac{24\epsilon}{bK^{m-1}} \km_X(C) 
 \end{align*}
\end{claim}
\begin{proof}
\begin{align*} 
&2\sum_{k=1}^K \sum_{x_n\in C_k}    \norm{\mu_k-\mu_k'} \norm{x_n-\mu_k} \\
&\leq \frac{2\epsilon}{b\kappa}  \sum_{k=1}^K \sum_{x_n\in C_k}    2\left( \norm{x-\mu_k}+ d(x_n,A) + R \right) \norm{x_n-\mu_k} \tag{Lemma~\ref{lem:candidates:mu-mu'}} \\
&\leq \frac{6\epsilon}{b\kappa}  \sum_{k=1}^K \sum_{x_n\in C_k}    \left( \norm{x-\mu_k}^2+ d(x_n,A)^2 + R^2 + \norm{x_n-\mu_k}^2 \right) \tag{Lemma~\ref{lem:squares}} \\
&\leq \frac{6\epsilon}{b\kappa} \left(  2\km_X(C) + \km_X(A) + NR^2  \right) \\
&\leq \frac{6\epsilon}{b\kappa} \left(  2\km_X(C) + 2\alpha\km_{X,K} \right) \tag{Lemma~\ref{lem:candidate:cK}}\\
&\leq \frac{24\epsilon}{b K^{m-1}} \km_X(C) \tag{since $\alpha \geq 1$}
\end{align*}
\end{proof}

\begin{claim}\label{claim:candidates:kmeans-approx:5}
 \begin{align*}
 2\sum_{k=1}^K \sum_{x_n\in C_k'}   \norm{\mu_k-\mu_k'} \norm{x_n-\mu_k'}
 \leq \frac{12\epsilon}{b\kappa} \left( \km_X(C')  + \alpha\cdot \km_X(C)\right)
 \end{align*}
\end{claim}
\begin{proof}
\begin{align*} 
&2\sum_{k=1}^K \sum_{x_n\in C_k'}    \norm{\mu_k-\mu_k'} \norm{x_n-\mu_k'} \\
&\leq \frac{2\epsilon}{b\kappa }  \sum_{k=1}^K \sum_{x_n\in C_k'}    2\left( \norm{x-\mu_k'}+ d(x_n,A) + R \right) \norm{x_n-\mu_k'} \tag{Lemma~\ref{lem:candidates:mu-mu'}} \\
&\leq \frac{6\epsilon}{b\kappa}  \sum_{k=1}^K \sum_{x_n\in C_k'}    \left( \norm{x-\mu_k'}^2+ d(x_n,A)^2 + R^2 + \norm{x_n-\mu_k'}^2 \right) \tag{Lemma~\ref{lem:squares}}\\
&\leq \frac{6\epsilon}{b\kappa} \left(  2\km_X(C') + \km_X(A) + NR^2  \right) \\
&\leq \frac{6\epsilon}{b\kappa} \left(  2\km_X(C') + 2\alpha\km_{X,K} \right) \tag{Lemma~\ref{lem:candidate:cK}}\\
&\leq \frac{12\epsilon}{b\kappa} \left( \km_X(C')  + \alpha\km_X(C)\right)
\end{align*}
\end{proof}

\begin{proof}[Proof of Claim~\ref{claim:candidates:kmeans-approx}]  
   
If $\km_X(C)> \km_X(C')$, then by Claim~\ref{claim:candidates:kmeans-approx:1}, \ref{claim:candidates:kmeans-approx:3}, and \ref{claim:candidates:kmeans-approx:5} we have
  \begin{align*}
    0
    &\leq \km_X(C)-\km_X(C')\\ 
    &\leq  \sum_{k=1}^K  \abs{C_k'} \norm{\mu_k-\mu_k'}^2 + 2 \sum_{x_n\in C_k'} \norm{\mu_k-\mu_k'} \norm{x_n-\mu_k'}^2\ \\
    &\leq  
    \frac{12\epsilon^2}{b^2\kappa^2}  \left(\km_X\left(C'\right) +   2\alpha \cdot\km_X(C)  \right) 
    +  \frac{12\epsilon}{b\kappa} \left( \km_X(C')  + 2\cdot \km_X(C)\right) \\
     &\leq  \frac{24\epsilon}{b\kappa}  \left( \km_X\left(C'\right) +   2\alpha \cdot\km_X(C) \right)
  \end{align*}
  Hence,
  \begin{align*}
     \left(1-\frac{48\alpha\epsilon}{b \kappa}\right) \km_X(C)- \left( 1+\frac{24\epsilon}{b\kappa}\right) \km_X(C')
     &\leq 0\\ 
      \left( 1+\frac{24\epsilon}{b\kappa }\right) \left( \km_X(C)- \km_X(C') \right)
     &\leq  \left(\frac{24\epsilon}{b\kappa} + \frac{48\alpha\epsilon}{b \kappa }\right)  \km_X(C) \\
       \km_X(C)- \km_X(C') 
     &\leq  \left(\frac{24\epsilon}{b\kappa} + \frac{48\alpha\epsilon}{b \kappa }\right)/ \left( 1+\frac{24\epsilon}{b\kappa }\right)  \km_X(C) \\
     &\leq  \frac{72\epsilon}{b K^{m-1} } \km_X(C) \ .
  \end{align*}

  If $\km(C')>\km(C)$, then by Claim~\ref{claim:candidates:kmeans-approx:1}, \ref{claim:candidates:kmeans-approx:2} and \ref{claim:candidates:kmeans-approx:4} we obtain   
  \begin{align*}
    0&\leq \km_X(C)-\km_X(C')\\
    &\leq \sum_{k=1}^K  \abs{C_k} \norm{\mu_k-\mu_k'}^2 + 2 \sum_{x_n\in C_k} \norm{\mu_k-\mu_k'} \norm{x_n-\mu_k}^2 \\
    &\leq \frac{36\epsilon^2}{b^2\kappa}\km_X(C) + \frac{24\epsilon}{bK^{m-1}}\km_X(C) \\
    &\leq \frac{60\epsilon}{b K^{m-1}}\km_X(C)
  \end{align*}
  
\end{proof}

\paragraph{Closeness with respect to the Fuzzy \texorpdfstring{$K$}{K}-Means Problem}\label{sec:proof:candidates:fuzzy-approx}

\begin{claim}\label{claim:candidate:fuzzy-approx}
	If $\alpha \leq 2$, then
 \[ \abs{\phiXm(C)-\phi_{X}^{(m)}(C')} \leq \frac{\epsilon}{4}\phiXm(C)\]
\end{claim}

In the following we prove Claim~\ref{claim:candidate:fuzzy-approx}. 
To this end, let $\{\rnk\}_{n,k}$ and $\{\rnk'\}_{n,k}$ be the optimal responsibilities with respect to $C$ and $C'$, respectively. 
Then, let
\begin{align*}
 \cE 
 \coloneqq \abs{\phiXm(C)-\phi_{X}^{(m)}(C')} 
 = \left\vert \sum_{n=1}^N \sum_{k=1}^K \rnk^m \norm{x_n-\mu_k}^2- (\rnk')^m \norm{x_n-\mu_k'}^2 \right\vert\ .
\end{align*}

\begin{claim}\label{claim:candidates:bound}
\begin{align*}
 \cE 
 &\leq \max\bigg\{ \sum_{n=1}^N \sum_{k=1}^K \rnk^m \norm{\mu_k-\mu_k'}^2+2\sum_{n=1}^N  \sum_{k=1}^K \rnk^m \norm{\mu_k-\mu_k'}  \norm{x_n-\mu_k}\ ,\\ 
 &\phantom{\leq \max\bigg\{}  \sum_{n=1}^N \sum_{k=1}^K (\rnk')^m \norm{\mu_k-\mu_k'}^2+2\sum_{n=1}^N \sum_{k=1}^K (\rnk')^m  \norm{\mu_k-\mu_k'}   \norm{x_n-\mu_k'}  \bigg\}\ .
\end{align*}
\end{claim}
\begin{proof}
  If the first term in $\cE$ is larger than the second, then 
\begin{align*}
 \cE 
 &= \sum_{n=1}^N \sum_{k=1}^K \rnk^m \norm{x_n-\mu_k}^2- (\rnk')^m \norm{x_n-\mu_k'}^2 \\
 &\leq \sum_{n=1}^N \sum_{k=1}^K (\rnk')^m \left( \norm{x_n-\mu_k}^2-  \norm{x_n-\mu_k'}^2\right) \tag{$\rnk$ optimal wrt. $C$} \\
 &\leq \sum_{n=1}^N \sum_{k=1}^K (\rnk')^m \left( \norm{\mu_k-\mu_k'}^2 +2  \norm{\mu_k-\mu_k'} \norm{x_n-\mu_k'}\right)\ , \tag{Lemma~\ref{lem:abc}}
\end{align*}
Analogously, if the second term in $\cE$ is larger than the first, then 
\begin{align*}
 \cE 
 &\leq \sum_{n=1}^N \sum_{k=1}^K \rnk^m \left( \norm{\mu_k-\mu_k'}^2 + 2     \norm{\mu_k-\mu_k'}   \norm{x_n-\mu_k} \right) \ .
\end{align*}
This yields the claim.
\end{proof}

\begin{claim}\label{claim:candidates:kmeans-approx:corollary}
 \[ \phiXm(C') \leq \gamma K^{m-1}\cdot \phiXm(C) \]
\end{claim}
\begin{proof}
 Using Claim~\ref{claim:candidates:kmeans-approx:corollary} and Lemma~\ref{lem:kmeans}, we obtain
 \[ \phiXm(C') \leq \km_X(C') \leq \gamma \km_X(C)\leq \gamma K^{m-1}\cdot \phiXm(C)\ .\]
\end{proof}

\begin{claim}\label{claim:candidates:mu-mu'-sum}
 \begin{align*}
   \sum_{k=1}^K \max\left\{ \sum_{n=1}^N \rnk^m, \sum_{n=1}^N (\rnk')^m \right\} \norm{\mu_k-\mu_k'}^2 \leq \frac{ 36 (\gamma+2\alpha) \epsilon^2}{b^2\alpha} \phiXm(C) 
 \end{align*}
\end{claim}
\begin{proof}
Observe that 
\begin{align*} 
  \sum_{n=1}^N  \sum_{k=1}^K \rnk^m  \norm{\mu_k-\mu_k'}^2 
  &\leq  \frac{12\epsilon^2}{b^2\kappa^2} \sum_{n=1}^N \sum_{k=1}^K \rnk^m  \left( \norm{x_n-\mu_k}^2 + d\left(x_n,A\right)^2 + R^2 \right)    \tag{Lemma~\ref{lem:candidates:mu-mu'}} \\
  &\leq  \frac{12\epsilon^2}{b^2\kappa^2} \left( \phiXm(C) + \sum_{n=1}^N d\left(x_n,A\right)^2 + NR^2 \right) \\ 
  &\leq  \frac{36\epsilon^2}{b^2\kappa} \phiXm(C) \tag{Lemma~\ref{lem:candidate:cK}}
\end{align*}
Similarly, 
\begin{align*} 
  \sum_{n=1}^N  \sum_{k=1}^K (\rnk')^m  \norm{\mu_k-\mu_k'}^2 
  &\leq  \frac{12\epsilon^2}{b^2\kappa^2} \sum_{n=1}^N \sum_{k=1}^K (\rnk')^m  \left( \norm{x_n-\mu_k'}^2 + d\left(x_n,A\right)^2 + R^2 \right)    \tag{Lemma~\ref{lem:candidates:mu-mu'}} \\
  &\leq  \frac{12\epsilon^2}{b^2\kappa^2} \left( \phiXm(C') + \sum_{n=1}^N d\left(x_n,A\right)^2 + NR^2 \right) \\
  &\leq  \frac{12\epsilon^2}{b^2\kappa^2} \left( \phiXm(C') +2\kappa \phiXm(C) \right) \tag{Lemma~\ref{lem:candidate:cK}}  \\
  &\leq  \frac{12\epsilon^2}{b^2\kappa^2} \left( \gamma K^{m-1} \phiXm(C) +2\kappa \phiXm(C) \right) \tag{Claim~\ref{claim:candidates:kmeans-approx:corollary}}\\
  &\leq  \frac{12 (\gamma+2\alpha)  \epsilon^2}{b^2\alpha^2} \phiXm(C)\ .
\end{align*}
\end{proof}

\begin{claim}\label{claim:candidates:mixed-sum}
 \begin{align*}
 2\sum_{n=1}^N \sum_{k=1}^K \rnk^m  \norm{\mu_k-\mu_k'} \norm{x_n-\mu_k}
 \leq  \frac{24 \epsilon}{b} \phiXm(C)
 \end{align*}
\end{claim}
\begin{proof}
\begin{align*} 
&2\sum_{n=1}^N  \sum_{k=1}^K \rnk^m \norm{\mu_k-\mu_k'}  \norm{x_n-\mu_k} \\
&\leq  \frac{2\epsilon}{b\kappa} \sum_{n=1}^N  \sum_{k=1}^K \rnk^m  2\left( \norm{\mu_k-x_n}+ d(x_n,A) + R \right)  \norm{x_n-\mu_k} \tag{Lemma~\ref{lem:candidate:cK}}\\ 
&\leq  \frac{2\epsilon}{b\kappa} \left( \sum_{n=1}^N  \sum_{k=1}^K \rnk^m  \left( \norm{\mu_k-x_n}+ d(x_n,A) + R \right)^2 +  \sum_{n=1}^N  \sum_{k=1}^K \rnk^m \norm{x_n-\mu_k}^2\right) \tag{Lemma~\ref{lem:squares}}\\ 
&\leq  \frac{6\epsilon}{b\kappa} \left(  \sum_{n=1}^N  \sum_{k=1}^K \rnk^m  \left( \norm{\mu_k-x_n}^2+ d(x_n,A)^2 + R^2 \right) + \phiXm(C) \right) \tag{Lemma~\ref{lem:squares}}\\ 
&\leq  \frac{6\epsilon}{b\kappa} \left(    \phiXm(C) + \sum_{n=1}^Nd(x_n,A)^2 + N R^2 + \phiXm(C) \right) \\ 
&\leq  \frac{24 \epsilon}{b} \phiXm(C) \tag{Lemma~\ref{lem:candidate:cK}} 
\end{align*}
\end{proof}

\begin{claim}\label{claim:candidates:mixed-sum'}
 \begin{align*}
 2\sum_{n=1}^N \sum_{k=1}^K (\rnk')^m  \norm{\mu_k-\mu_k'} \norm{x_n-\mu_k'}
 \leq \frac{12 (\gamma+\alpha) \epsilon}{b\alpha} \phiXm(C)
 \end{align*}
\end{claim}
\begin{proof}
\begin{align*} 
&2\sum_{n=1}^N  \sum_{k=1}^K (\rnk')^m \norm{\mu_k-\mu_k'}  \norm{x_n-\mu_k'} \\
&\leq  \frac{2\epsilon}{b\kappa} \sum_{n=1}^N  \sum_{k=1}^K (\rnk')^m  2\left( \norm{\mu_k'-x_n}+ d(x_n,A) + R \right)  \norm{x_n-\mu_k'} \tag{Lemma~\ref{lem:candidate:cK}}\\ 
&\leq  \frac{2\epsilon}{b\kappa} \left( \sum_{n=1}^N  \sum_{k=1}^K (\rnk')^m  \left( \norm{\mu_k'-x_n}+ d(x_n,A) + R \right)^2 +  \sum_{n=1}^N  \sum_{k=1}^K (\rnk') \norm{x_n-\mu_k'}^2\right) \tag{Lemma~\ref{lem:squares}}\\ 
&\leq  \frac{6\epsilon}{b\kappa} \left(  \sum_{n=1}^N  \sum_{k=1}^K (\rnk')^m  \left( \norm{\mu_k'-x_n}^2+ d(x_n,A)^2 + R^2 \right) + \phiXm(C') \right) \tag{Lemma~\ref{lem:squares}}\\ 
&\leq  \frac{6\epsilon}{b\kappa} \left(    \phiXm(C') + \sum_{n=1}^Nd(x_n,A)^2 + N R^2 + \phiXm(C') \right)\\ 
&\leq  \frac{6 \epsilon}{b\kappa}\left(  2\phiXm(C') + 2\kappa \phiXm(C) \right)  \tag{Lemma~\ref{lem:candidate:cK}}\\ 
&\leq  \frac{12 (\gamma+\alpha) \epsilon}{b\alpha}\phiXm(C)  \tag{Claim~\ref{claim:candidates:kmeans-approx:corollary}}
\end{align*}
\end{proof}

\begin{proof}[Proof of Claim~\ref{claim:candidate:fuzzy-approx}]
\begin{align*}
 \cE 
 &\leq \max\bigg\{ \sum_{n=1}^N \sum_{k=1}^K \rnk^m \norm{\mu_k-\mu_k'}^2+2\sum_{n=1}^N  \sum_{k=1}^K \rnk^m \norm{\mu_k-\mu_k'}  \norm{x_n-\mu_k}\ ,\\ 
 &\phantom{\leq \max\bigg\{}  \sum_{n=1}^N \sum_{k=1}^K (\rnk')^m \norm{\mu_k-\mu_k'}^2+2\sum_{n=1}^N \sum_{k=1}^K (\rnk')^m  \norm{\mu_k-\mu_k'}   \norm{x_n-\mu_k'}  \bigg\} \\
 &\leq \max\left\{\sum_{n=1}^N \sum_{k=1}^K \rnk^m \norm{\mu_k-\mu_k'}^2, \sum_{n=1}^N \sum_{k=1}^K (\rnk')^m \norm{\mu_k-\mu_k'}^2\right\} + \\
 &\phantom{\leq\ } \max\left\{2\sum_{n=1}^N  \sum_{k=1}^K \rnk^m \norm{\mu_k-\mu_k'}  \norm{x_n-\mu_k} , 2\sum_{n=1}^N \sum_{k=1}^K (\rnk')^m  \norm{\mu_k-\mu_k'}   \norm{x_n-\mu_k'}\right\} \\
 &\leq \frac{ 36 (\gamma+2\alpha) \epsilon^2}{b^2\alpha} \phiXm(C) + \frac{24(\gamma+\alpha)\epsilon}{b} \phiXm(C) \tag{By Claims~\ref{claim:candidates:mu-mu'-sum}, \ref{claim:candidates:mixed-sum} and \ref{claim:candidates:mixed-sum'}} \\
 &\leq \frac{60(\gamma+2\alpha)\epsilon}{b} \phiXm(C)
\end{align*}
where $\gamma = 1+\frac{72\epsilon}{bK^{m-1}}$ (cf. Claim~\ref{claim:candidates:kmeans-approx}). 
Since $b=1208$ and $\alpha \leq 2$, we have $\cE \leq \frac{\epsilon}{4}\phiXm(C)$, which yields the claim. 
\end{proof}

\subsubsection{Upper Bound on the Size of \texorpdfstring{$\cG$}{G}}

In the following, we upper bound $\abs{\cG}$ analogously to \cite{Chen09}.

Recall from Section~\ref{sec:proof:candidates:construction} that each $L_{k,j}$ is partitioned into an axis-parallel grid with side length $\rho_j = \frac{2^j \epsilon R}{b\kappa \sqrt{D}}$.
Hence, the volume of each grid cell is 
\[ V_j = \left(\frac{2^j \epsilon R}{b\kappa \sqrt{D}}\right)^D .\]

Furthermore, observe that the distance between a point $x\in L_{k,j}$ and mean vector $a_k$ is at most $2^jR+\rho_j < 2^{j+1}R$. 
Hence, the grid cell that contains $x$ is contained in $\ball(a_k,2^{j+1}R)$. 
Each of these balls has a volume of 
\[ \vol(\ball(a_k,2^{j+1}R)) = \frac{\pi^{\nicefrac{D}{2}} (2^{j+1}R)^D}{\Gamma(\nicefrac{D}{2}+1)} . \]

Consequently, the number of grid cells in each $L_{k,j}$ is bounded by
\begin{align*}
	\frac{\vol(\ball(a_k,2^{j+1}R))}{V_j}
	&= 
	\left( \frac{\pi^{\nicefrac{D}{2}} (2^{j+1}R)^D}{\Gamma(\nicefrac{D}{2}+1)}\right) 
	\left(\frac{b\kappa \sqrt{D}}{2^j \epsilon R}\right)^D \\
	&\leq 
	\left( \frac{\pi^{\nicefrac{D}{2}} (2^{j+1}R)^D}{(\nicefrac{D}{4e})^{\nicefrac{D}{2}}}\right)
	\left(\frac{b\kappa \sqrt{D}}{2^j \epsilon R}\right)^D \\
	&= 
	\left( \frac{2b\kappa}{\epsilon}\right)^D\left(4e\pi\right)^{\nicefrac{D}{2}} 
	\leq \left( \frac{12b\kappa}{\epsilon}\right)^D\ .
\end{align*}
Overall, we obtain
\begin{align*}
	\abs{\cG} 
	&\leq \sum_{j=0}^\Phi \sum_{k=1}^K \frac{\vol(\ball(a_k,2^{j+1}R))}{V_j} \\
	&\leq K \left(\log(\alpha N) + m \log\left( \frac{64\alpha m K^2}{\epsilon}\right)\right)\left( \frac{12b\kappa}{\epsilon}\right)^D \\
	&= \cO\left( K^{mD+1} \epsilon^{-D}m\log\left(\nicefrac{mK}{\epsilon}\right)\log(N)\right) \ .
\end{align*}

%% file: app_solv.tex
\subsection{Unsolvability by Radicals (Proof of Theorem~\ref{thm:radicals})}\label{app:radicals}
	Consider the fuzzy $2$-means instance with $m=2$ and $X =\{-3,-2,-1,1,2,3\}\subset\R$.
	Let $\{\mu_1^*, \mu_2^*\}\subset \R$ be the means of an optimal solution.
	
	\begin{claim}\label{claim:solv:signs}
		$\sgn(\mu^*_1) \neq \sgn(\mu^*_2)$
	\end{claim}
	\begin{proof}
		Assume w.l.o.g. that $\mu^*_1, \mu^*_2 \leq 0$. 
		Let $\rnkSet$ be the memberships induced by $\{\mu^*_1, \mu^*_2 \}$.
		Using the Cauchy-Schwarz inequality, we conclude
		\[ 
		\phi_{(X,2,2)}^{OPT}
		= \phi_X^{(2)}(\mu^*_1, \mu^*_2) 
		\geq r_{31}^2(3 - \mu_1)^2 + r_{32}^2(3-\mu_2)^2 
		\geq \frac{1}{2} 3^2 > 4\ .
		\]
		This contradicts the fact that, due Lemma~\ref{lem:kmeans}, we have $\phi_{(X,2,2)}^{OPT} \leq \km_X(\{2,-2\}) = \sum_{x\in X} \min\{(x-2)^2, (x+2)^2\} = 4 $. 
		This yields the claim.
	\end{proof}
	
	\begin{observation}
	  $\mu_1^*$ and $\mu_2^*$ lie inside the convex hull of the point set $X$. 
	\end{observation}

	From this observation and Claim~\ref{claim:solv:signs}, we can conclude that the optimal solution $\{\mu_1^*, \mu_2^*\}$ satisfies the following equations: 
	\begin{align*}
		\frac{\partial \phi_X^{(2)}(\{\mu_1, \mu_2^*\})}{\partial \mu_1}(\mu_1^*) = 0 \;&,\; \frac{\partial \phi_X^{(2)}(\{\mu_1^*, \mu_2\})}{\partial \mu_2}(\mu_2^*) = 0 \\
		3 \geq \mu_1^* > 0 \;&,\; 0 > \mu_2^* \geq -3 .
	\end{align*}
	One can check that the only pair of real values satisfying all of the equations above are the two real roots of the polynomial
	\[ g(x) = 3x^{12}+84x^{10}+490x^8-292x^6-8981x^4-17640x^2-11664 . \]
	The interested reader can reproduce this result using the CAS Maple\texttrademark{}\footnote{Maple is a trademark of Waterloo Maple Inc.} and the worksheet provided in Section~\ref{subsec:mapleworksheet}.
	
	Note, that the roots of the polynomial
	\[ h(x) = 3x^6+84x^5+490x^4-292x^3-8981x^2-17640x-11664 \]
	are the square roots of the roots of $g$.
	By using the following well-known results from algebra, we can show that the roots of $h$, and hence also the roots of $g$, are not solvable by radicals over $\Q$.
	\begin{definition}
		We call a prime $p$ \emph{good} for a polynomial $f\in \Q[x]$ if $p$ does not divide the discriminant of $f$.
	\end{definition}
	\begin{lemma}[\cite{bajaj88}]\label{lem:cycles}
		Let $f\in \Q[x]$ with $\deg(f) = n > 2$ and $\deg(f) = 0 \mkern-6mu \mod 2$.
		If there are good primes $p_1, p_2, p_3$ for $f$ such that
		\begin{enumerate}
			\item $f \mkern-6mu \mod p_1$ is an irreducible polynomial of degree $n$,
			\item $f \mkern-6mu \mod p_2$ factors into a linear polynomial and an irreducible polynomial of degree $n-1$, and
			\item $f \mkern-6mu \mod p_3$ factors into a linear polynomial, an irreducible polynomial of degree $2$ and an irreducible polynomial of degree $n-3$,
		\end{enumerate}
		then $\Gal(f) \cong S_n$.
	\end{lemma}
	\begin{lemma}[\cite{hungerford74}]\label{lem:groups}
		Let $f\in \Q[x]$.
		If the equation $f(x) = 0$ is solvable by radicals over $\Q$, then the Galois group of $f$ is a solvable group.
	\end{lemma}
	\begin{lemma}[\cite{hungerford74}]\label{lem:S_n}
		The symmetric group $S_n$ is not solvable for $n\geq 5$ .
	\end{lemma}
	\begin{corollary}
		The equation $h(x) = 0$ is not solvable by radicals over $\Q$.
	\end{corollary}
	\begin{proof}
		Since the discriminant of $h$ is $D(h) = 2^{31}\cdot 3^7 \cdot 5^2 \cdot 7^3 \cdot 76637866514129$, we can conclude that $11$, $17$, and $89$ are good primes for $h$.
		We factor $h$ modulo these good primes
		\begin{align*}
			h &= 3\cdot (x^6+6x^5+2x^4+9x^3+2x^2+5x+6) \mod 11 \\
			h &= 3\cdot (x^5+3x^4+9x^3+12x^2+10x+7)\cdot (x+8) \mod 17\\
			h &= 3\cdot (x^3+17x^2+50x+17)\cdot (x^2+9x+27)\cdot (x+2) \mod 89 .
		\end{align*}
		From Lemma~\ref{lem:cycles} we obtain $\Gal(h) \cong S_6$.
		Applying Lemmata~\ref{lem:groups} and \ref{lem:S_n} yields the claim.
	\end{proof}

\subsection{Maple\texttrademark Worksheet}\label{subsec:mapleworksheet}
The following worksheet was developed using Maple\texttrademark{} 13.0.
It can be downloaded at \url{https://www-old.cs.uni-paderborn.de/index.php?id=45441}.
In the worksheet we use the following formulation of the fuzzy $2$-means objective function with $m=2$.
\begin{observation}
       For all $\muSet\subset\R^D$ and $X=\{x_n\}_{n\OneTo{N}}\subset\R^D$ with $x_n\neq \mu_k$ for all $k\OneTo{K}$ and $n\OneTo{N}$, we have
	\[ \phiXm(\muSet) = \sum_{n=1}^N \sum_{k=1}^K \left(\frac{\norm{x_n-\mu_k}^{-2}}{\sum_{l=1}^K \norm{x_n-\mu_l}^{-2}}\right)^2 \norm{x_n-\mu_k}^2 = \sum_{n=1}^N \frac{1}{\sum_{k=1}^K \norm{x_n-\mu_k}^{-2}} , \]
	where the first equality is due to Equation~\ref{eq:opt-membership}. 
\end{observation}

\begin{figure}[htbp]
  \begin{framed}
		\noindent $>${\color{red}\textit{with(RealDomain):}}
		\vspace*{12pt}
		
		\noindent $>${\color{red}$\phi \coloneqq \frac{1}{(3-\mu_1)^{-2} + (3-\mu_2)^{-2}} + \frac{1}{(2-\mu_1)^{-2} + (2-\mu_2)^{-2}} + \frac{1}{(1-\mu_1)^{-2} + (1-\mu_2)^{-2}} +$}
		
		\noindent \phantom{$>$}{\color{red}$\frac{1}{(-3-\mu_1)^{-2} + (-3-\mu_2)^{-2}} + \frac{1}{(-2-\mu_1)^{-2} + (-2-\mu_2)^{-2}} + \frac{1}{(-1-\mu_1)^{-2} + (-1-\mu_2)^{-2}}$}
		\vspace*{12pt}
		
		\noindent $>${\color{red}$sol:=solve(\{diff(\phi, \mu_1) = 0, diff(\phi, \mu_2) = 0, 3 >= \mu_1, \mu_1 > 0, 0 > \mu_2 , \mu_2 >= -3\},$}
		
		\noindent \phantom{$>$}{\color{red}$[\mu_1, \mu_2])$}
		\vspace*{12pt}
		
		{\color{blue}$[[\mu_1 = RootOf(3\_Z^{12}+84\_Z^{10}+490\_Z^8-292\_Z^6-8981\_Z^4-17640\_Z^2-11664,$
		\vspace*{6pt}
		
		$2.032093935),\mu_2 = -RootOf(3\_Z^{12}+84\_Z^{10}+490\_Z^8-292\_Z^6-8981\_Z^4-$
		\vspace*{6pt}
		
		$17640\_Z^2-11664, 2.032093935)]]$}
 \end{framed}
	\caption{Our Maple\texttrademark Worksheet}
\end{figure}

%% file: app_bad-example.tex
\subsection{Arbitrarily Poor Local Minima (Proof of Observation~\ref{obs:poor-solution:intro})} \label{sec:poor-solutions} 

It is known that the FM algorithm converges to a stationary point of the objective function that is either a saddlepoint or a (local) minimum \cite{bezdeck87}. 
We show that there are instances for which this point is arbitrarily poor compared to an optimal solution.  

\begin{claim}\label{claim:poor-solution}
 Let $m\in\N$, $D\geq 2$, and $K=2$. 
 Choose an arbitrary $c \in\R$. 
 Then, there exists a point set $X_a$ and initial point set $I\subset X_a$, $\abs{I}=2$, satisfying the following properties:
 If the FM algorithm is initialized with $I$, then in each round it computes a solution whose cost are at least $c\cdot \phiopt$. 
\end{claim}
\begin{proof}
Consider the unweighted instances
\[ X_{a} \coloneqq \{(a,1), (-a,1),  (-a,-1), (a,-1)\}\subset \R^2 , \]
where $a\in\R$ with $a>1$. 

\begin{claim}\label{claim:poor-solution:lower-bound-rnk-sum}
 Let $\rnkSet$ be a solution to the fuzzy $2$-means problem with respect to $X_a$. 
 Then, for all $x_n\in X_a$ we have $r_{n1}^m+r_{n2}^m\geq \left(\frac{1}{2}\right)^m$. 
\end{claim}
\begin{proof}
 Since $r_{n1}+r_{n2}=1$, we know $\max\{r_{n1},r_{n2}\}\geq 1/2$ and thus $r_{n1}^m+r_{n2}^m\geq \left(\frac{1}{2}\right)^m$. 
\end{proof}

\begin{claim}\label{claim:poor-solution:opt}
 An optimal fuzzy $2$-means clustering of $X_a$ costs at least $\frac{1}{2^{m-1}}$ and at most $4$.
\end{claim}
\begin{proof}
 Observe that the means of every optimal solution lie in the convex hull of the input points.
 Consider an arbitrary solution $\{\mu_1,\mu_2\}$ inside the rectangle spanned by $X_a$.
 There are two points in $X_a$ for which the distance to both means is at least $1$.
 Using Claim~\ref{claim:poor-solution:lower-bound-rnk-sum}, we conclude that the costs of any solution can be lower bounded by $2\cdot \left(\frac{1}{2}\right)^m\cdot 1 = \frac{1}{2^{m-1}}$.
 
 Finally, observe that for $\mu_1 = (-a, 0)$, $\mu_2 = (a, 0)$ we have 
  $
		\phi_{(X_a,2,m)}^{OPT} < \norm{x_1 - \mu_2}^2 + \norm{x_2 - \mu_2}^2 + \norm{x_3 - \mu_2}^2 + \norm{x_4 - \mu_1}^2 = 4
  $.
\end{proof}

However, the FM algorithm might compute arbitrarily poor solutions, even if it is initialized with points from the point set.

\begin{claim}\label{claim:poor-solution:factor}
 If the FM algorithm is started on $X_{a}$ with $\{(a,1),(a,-1)\}$ as initial centers, then it computes a solution that has at least cost $\frac{a^2}{2^{m+1}}\phi^{(OPT)}_{(X_a,2,m)}$.
\end{claim}
\begin{proof}
        Let $x_1 = (a,1)$, $x_2 = (-a,1)$, $x_3=-x_1$, $x_4=-x_2$, $\mu_1 = (a,1)$, and $\mu_2 = (a,-1)$.

	First, we show that if the FM algorithm is initialized with means $(\tmu_1,\tmu_2)$ that lie on a line parallel to the $y$-axis, then it computes means that also lie on a line parallel to the $y$-axis.	
	Given $(\tmu_1,\tmu_2)$, the algorithm computes memberships where
	$r_{11} = r_{42}$, $r_{21} = r_{32}$, $r_{31}=r_{22}$ and $r_{41}=r_{12}$.
	Hence,
	\begin{align*}
		(\tmu_1)_x 
		= \frac{r_{11}^m a - r_{21}^m a - r_{31}^m a + r_{41}^m a}{r_{11}^m+r_{21}^m+r_{31}^m+r_{41}^m}
		= \frac{r_{42}^m a - r_{32}^m a - r_{22}^m a + r_{12}^m a}{r_{42}^m+r_{32}^m+r_{22}^m+r_{12}^m}
		= (\tmu_2)_x \enspace .
	\end{align*}

	Next, we lower bound the cost of means $\{\tmu_1,\tmu_2\}$ that lie on a line parallel to the $y$-axis.
	Observe that there are always at least $2$ points in $X_{a}$ that have distance at least $a^2$ from both means.
	Without loss of generality, we can assume that these points are $x_2$ and $x_3$.  
	Denote by $\rnkSet$ the optimal responsibilities induced by $\{\tmu_1,\tmu_2\}$. 
	Then,  
	\begin{align*}
		\phi(\{\tmu_1,\tmu_2\} 
		= \sum_{n=1}^4 \sum_{k=1}^2 r_{nk}^m \norm{x_n - \tmu_k}^2
		\geq a^2 \sum_{k=1}^2 r_{2k}^m + r_{3k}^m
		\geq \frac{a^2}{2^{m-1}}
		\geq \frac{a^2}{2^{m+1}}\phi^{(OPT)}_{(X_a,2,m)} \enspace ,
	\end{align*}
	where the second last inequality follows from Claim~\ref{claim:poor-solution:lower-bound-rnk-sum} and the last inequality follows from Claim~\ref{claim:poor-solution:opt}. 
\end{proof}

  Applying Claim~\ref{claim:poor-solution:factor} with $a\coloneqq \lceil 2^m \sqrt{c} \rceil$ yields the claim. 
\end{proof}